\documentclass{article}

\usepackage[preprint]{neurips_data_2023}

\usepackage[utf8]{inputenc} 
\usepackage[T1]{fontenc}    
\usepackage{hyperref}       
\usepackage{url}            
\usepackage{booktabs}       
\usepackage{amsfonts}       
\usepackage{nicefrac}       
\usepackage{microtype}      
\usepackage{xcolor}         

\usepackage{amsmath}
\usepackage{amssymb}
\usepackage{amsthm}
\usepackage{multirow}
\usepackage{paralist}
\usepackage{graphicx}
\usepackage{mathabx}
\usepackage{diagbox}
\usepackage{algorithm}
\usepackage{algorithmic}

\usepackage[textsize=tiny]{todonotes}

\theoremstyle{plain}
\newtheorem{theorem}{Theorem}
\newtheorem{proposition}[theorem]{Proposition}

\newtheorem{definition}[theorem]{Definition}
\newtheorem{assumption}[theorem]{Assumption}

\newcommand{\entity}{\mathcal{E}}
\newcommand{\relation}{\mathcal{R}}
\newcommand{\efok}{\textsc{EFO}_k}

\title{$\efok$-CQA: Towards Knowledge Graph \\ Complex Query Answering beyond Set Operation}

\author{Hang Yin\\
  Department of Mathematical Sciences\\
  Tsinghua University \\
  \texttt{h-yin20@mails.tsinghua.edu.cn}\\
   \And
   Zihao Wang \\
  Department of CSE\\
  HKUST\\
  \texttt{zwanggc@cse.ust.hk} \\
  \And
  Weizhi Fei \\
  Department of Mathematical Sciences\\
  Tsinghua University \\
  \texttt{fwz22@mails.tsinghua.edu.cn} \\
  \And
  Yangqiu Song\\
  Department of CSE\\
  HKUST\\
  \texttt{yqsong@cse.ust.hk}}

\begin{document}

\maketitle

\begin{abstract}
To answer complex queries on knowledge graphs, logical reasoning over incomplete knowledge is required due to the open-world assumption. Learning-based methods are essential because they are capable of generalizing over unobserved knowledge. Therefore, an appropriate dataset is fundamental to both obtaining and evaluating such methods under this paradigm. In this paper, we propose a comprehensive framework for data generation, model training, and method evaluation that covers the combinatorial space of Existential First-order Queries with multiple variables ($\efok$). The combinatorial query space in our framework significantly extends those defined by set operations in the existing literature. Additionally, we construct a dataset, $\efok$-CQA, with 741  types of query for empirical evaluation, and our benchmark results provide new insights into how query hardness affects the results. Furthermore, we demonstrate that the existing dataset construction process is systematically biased that hinders the appropriate development of query-answering methods, highlighting the importance of our work. Our code and data are provided in~\url{https://github.com/HKUST-KnowComp/EFOK-CQA}.
\end{abstract}

\section{Introduction}
The Knowledge Graph (KG) is a powerful database that encodes relational knowledge into a graph representation~\cite{vrandecic_wikidata_2014,suchanek_yago_2007}, supporting downstream tasks~\cite{zhou_bipartite_2007,ehrlinger_towards_2016} with essential factual knowledge. However, KGs suffer from incompleteness during its construction~\cite{vrandecic_wikidata_2014,carlson_toward_2010}, which is formally acknowledged as Open World Assumption (OWA)~\cite{libkin_open_2009}. 
The task of Complex Query Answering (CQA) proposed recently has attracted much research interest~\cite{hamilton_embedding_2018,ren_beta_2020}.
This task ambitiously aims to answer database-level complex queries described by logical complex connectives (conjunction $\land$, disjunction $\lor$, and negation $\lnot$) and quantifiers\footnote{The universal quantifier is usually not considered in query answering tasks, as a common practice from both CQA on KG~\cite{wang_logical_2022,ren_neural_2023} and database query answering~\cite{poess_new_2000}} (existential $\exists$)~\cite{wang_logical_2022,ren_neural_2023,leskovec_databases_2023}.
However, CQA on KGs differs from query answering on databases in two aspects: (1) traditional query answering algorithms obtain incomplete answers because of the incomplete KG~\cite{hamilton_embedding_2018}; (2) the huge size of the knowledge graph limits the scalability of traditional algorithms~\cite{ren_query2box_2020}.
Therefore, learning-based methods dominate the CQA tasks because they can empirically generalize to unseen knowledge as well as prevent the resource-demanding symbolic search.

The thriving of learning-based methods also puts an urgent request on high-quality datasets and benchmarks. In the previous study, datasets are developed by progressively expanding the \textbf{syntactical expressiveness}, where conjunction~\cite{hamilton_embedding_2018}, union~\cite{ren_query2box_2020}, negation~\cite{ren_beta_2020}, and other operators~\cite{liu_neural-answering_2021} are taken into account sequentially. In particular, the dataset proposed in~\cite{ren_beta_2020} contains all logical connectives and becomes the standard training set for model development.~\cite{wang_benchmarking_2021} proposed a large evaluation benchmark EFO-1-QA that systematically evaluates the combinatorial generalizability of CQA models on such queries. More related works are included in Appendix~\ref{app: related works}.

However, the queries in aforementioned datasets~\cite{ren_beta_2020,wang_benchmarking_2021} are recently justified as ``Tree-Form'' queries~\cite{yin_existential_2023} as they rely on the tree combinations of set operations. Compared to the well-established TPC-H decision support benchmark~\cite{poess_new_2000} for database query processing, queries in existing CQA benchmarks~\cite{ren_beta_2020,wang_benchmarking_2021} have two common shortcomings: (1) lack of \textbf{combinatorial answers}: only one variable is queried, and (2) lack of \textbf{structural hardness}: all existing queries are subject to the structure-based tractability~\cite{rossi_handbook_2006,yin_existential_2023}. It is rather questionable whether existing CQA data under such limited scope can support the future development of methodologies for general decision support with open-world knowledge.

The goal of this paper is to establish a new framework that addresses the aforementioned shortcomings to support further research in complex query answering on knowledge graphs. Our framework is formally motivated by the well-established investigation of constraint satisfaction problems, which all queries can be formulated as. In general, the contribution of our work is four folds.
\begin{compactdesc}
    \item[Complete coverage] We capture the complete Existential First Order (EFO) queries from their rigorous definitions, underscoring both \textbf{combinatorial hardness} and \textbf{structural hardness} and extending the existing  coverage~\cite{wang_benchmarking_2021} which covers only a subset of $\textsc{EFO}_1$ query. The captured query family is denoted as $\efok$ where $k$ stands for multiple variables. 
    \item[Curated datasets] We derive $\efok$-CQA dataset, a systematic extension of the previous EFO-1-QA benchmark~\cite{wang_benchmarking_2021} and contains 741 types of query. We design several rules to guarantee that our dataset includes high-quality nontrivial queries, particularly those that contain multiple query variables and are not structure-based tractable.
    \item[Convenient implementation] We implement the entire pipeline for query generation, answer sampling, model training and inference, and evaluation for the undiscussed scenarios of \textbf{combinatorial answers}. Our pipeline is backward compatible, which supports both set operation-based methods and more recent ones.
    \item[Results and findings] We evaluate six representative CQA methods on our benchmark. Our results refresh the previous empirical findings and further reveal the structural bias of previous data.
\end{compactdesc}

\section{Problem definition}

\subsection{Existential first order (EFO) queries on knowledge graphs}\label{sec:knowledge graph definition}

Given a set $\entity$ of entities and a set $\relation$ of relations , a knowledge graph  $\mathcal{KG}$ encodes knowledge as set of factual triple $\mathcal{KG} = \{(h, r, t)\}\subset \mathcal{V} \times \mathcal{R} \times \mathcal{V}$. According to the OWA, the knowledge graph that we have observed $\mathcal{KG}_{o}$ is only part of the real knowledge graph, meaning that $\mathcal{KG}_{o} \subset \mathcal{KG}$.

The existing research only focuses on the logical formulas without universal quantifiers~\cite{ren_neural_2023,wang_logical_2023}. We then offer the definition of it based on strict first order logic.

\begin{definition}[Term]
    A term is either a variable $x$ or an entity $a\in \entity$.
\end{definition}

\begin{definition}[Atomic formula]\label{def:atomic-formula}
    $\phi$ is an atomic formula if $\phi= r(h, t)$, where $r\in \relation$ is a relation, $h$ and $t$ are two terms. 
\end{definition}

\begin{definition}[Existential first order formula]\label{def:formula}
The set of the existential formulas is the smallest set $\Phi$ that satisfies the following:
\begin{compactitem}
    \item[(i)] For atomic formula $r(h,t)$, itself and its negation $r(h,t), \lnot r(h,t)\in \Phi$
    \item[(ii)] If $\phi, \psi\in \Phi$, then $(\phi\land \psi),(\phi \lor \psi) \in \Phi$
    \item[(iii)] If $\phi \in \Phi$ and $x_i$ is any variable, then $\exists x_i \phi \in \Phi$.
\end{compactitem}
\end{definition}

\begin{definition}[Free variable]
    If a variable $y$ is not associated with a quantifier, it is called a free variable, otherwise, it is called a bounded variable. We write $\phi(y_1,\cdots, y_k)$ to indicate $y_1,\cdots,y_k$ are the free variables of $\phi$.
\end{definition}

\begin{definition}[Sentence and query]
    A formula $\phi$ is a sentence if it contains no free variable, otherwise, it is called a query. In this paper, we always consider formula with free variable, thus, we use formula and query interchangeably.
\end{definition}

\begin{definition}[Substitution]
    For $a_1,\cdots,a_k$, where $a_i\in \entity$, we write $\phi(a_1/y_1,\cdots, a_k/y_k)$ or simply $\phi(a_1,\cdots, a_k)$ for the result of simultaneously replacing all free occurrence of $y_i$ in $\phi$ by $a_i$, $i=1,\cdots,k$.
\end{definition}

\begin{definition}[Answer of an EFO query]
    For a given existential query $\phi(y_1,\cdots,y_k)$, its answer is a set that defined by
    \begin{align*}
        \mathcal{A}[\phi(y_1,\cdots,y_k)] = \{(a_1,\cdots,a_k) |~a_i\in \entity, i=1,\cdots,k \text{, }\phi(a_1,\cdots, a_k) \text{ is True}\}
    \end{align*}
\end{definition}

\begin{definition}[Disjunctive Normal Form (DNF)] For any existential formula $\phi(y_1,\cdots,y_k)$, it can be converted to the disjunctive normal form as shown below:
\begin{align}
    \phi(y_1,\cdots,y_k) &= \gamma_1(y_1,\cdots,y_k) \lor  \cdots \lor \gamma_m(y_1,\cdots,y_k)
    \\
    \gamma_i(y_1,\cdots,y_k) &= \exists x_1,\cdots,x_n. \rho_{i1} \land \cdots \land \rho_{it} \label{def:DNF formula}
\end{align}
 where $\gamma_i$ is called conjunctive formulas, $\rho_{ij}$ is either an atomic formula or the negation of an atomic formula, $x_i$ is called an existential variable.
\end{definition}

DNF has a strong property that $\mathcal{A}[\phi(y_1,\cdots,y_k)] = \cup_{i=1}^{m} \mathcal{A}[\gamma_i(y_1,\cdots,y_k)]$, which allows us to only consider conjunctive formulas $\gamma_i$ and then aggregate those answers to retrieve the final answers. This practical technique has been used in many previous research~\cite{long_neural-based_2022,ren_neural_2023}. Therefore, we only discuss conjunctive formulas in the rest of this paper.

\begin{figure*}[t]
\centering
\includegraphics[width=.95\linewidth]{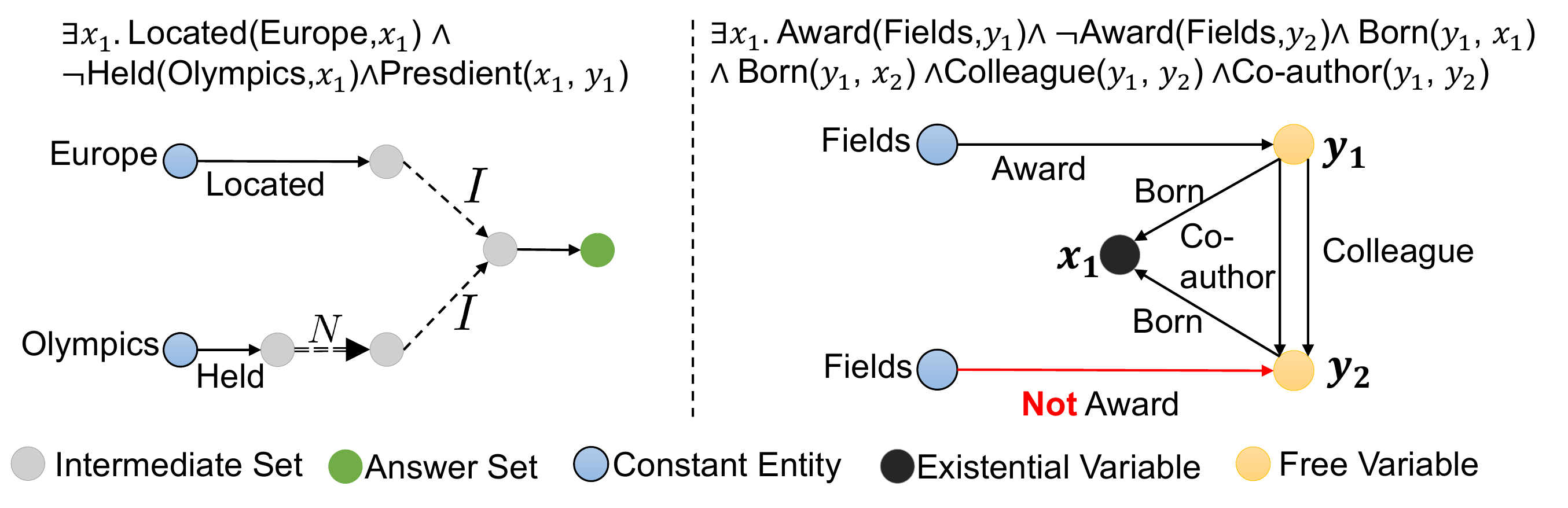}
\caption{Operator Tree versus Query Graph. \textbf{Left}: An operator tree representing  a given query ``List the presidents of European countries that have never held the Olympics'' \cite{ren_beta_2020}; \textbf{Right}: A query graph representing a given query ``Find a pair of persons who are both colleagues and co-authors and were born in the same country, with one having awarded the fields medal while the another not'', which is both a multigraph and a cyclic graph, containing two free variables.}
\label{fig:represent_query}
\vspace{-1.5em}
\end{figure*}

\subsection{Constraint satisfaction problem for conjunctive queries}\label{sec: conjunctive formula as CSP}
Formally, a constraint satisfaction problem (CSP) $\mathcal{P}$ can be represented by a triple $\mathcal{P}=(X,D,C)$ where $X=(x_1,\cdots, x_n)$ is an $n$-tuple of variables, $D=(D_1,\cdots, D_n)$ is the corresponding $n$-tuple of domains, $C=(C_1,\cdots,C_t)$ is $t$-tuple of constraints, each constraint $C_i$ is a pair of $(S_i,R_{S_i})$ where $S_i$ is a set of variables $S_i=\{x_{i_j}\}$ and $R_{{S_i}}$ is the constraint over those variables~\cite{rossi_handbook_2006}.

Historically, there are strong parallels between CSP and conjunctive queries in knowledge bases~\cite{gottlob_hypertree_1999,kolaitis_conjunctive-query_1998}. The terms correspond to the variable set $X$. The domain $D_i$ of a constant entity contains only itself, while it is the whole entity set $\entity$ for other variables. Each constraint $C_i$ is binary that is induced by an atomic formula or its negation, for example, for an atomic formula $r(h,t)$,  we have $S_i = \{h,t\}$, $R_{S_i}=\{(h,t) | h,t\in \entity, (h,r,t)\in \mathcal{KG}\}$. Finally, by the definition of existential quantifier, we only consider the answer of free variable, rather than tracking all terms within the existential formulas.

\begin{definition}[CSP answer of conjunctive formula]\label{def: CSP answer of existential formula} For a conjunctive formula $\gamma$ in Equation~\ref{def:DNF formula} with $k$ free variables and $n$ existential variables, the answer set of it formulated as CSP instance is:
\begin{align*}
    \mathcal{\overline{A}}[\gamma(y_1,\cdots,y_k)] = \mathcal{A}[\gamma^{\star}(y_1,\cdots,y_{n+k})], \text{ where } \gamma^{\star} = \rho_{i1} \land \cdots \land \rho_{it}
\end{align*}
\end{definition}

This shows that the inference of existential formulas is easier than solving CSP instances since the existential variables do not need to be kept track of.

\subsection{The representation of query}\label{sec: representation of formula}

To give an explicit representation of EFO query,~\cite{hamilton_embedding_2018} firstly proposes to represent a query by operator tree, where each node represents the answer set for a sub-query, and the logic operators in it naturally represent set operations. This method allows for the recursive computation from constant entity to the final answer set in a bottom-up manner~\cite{ren_beta_2020}. However, this representation method is inherently directed, acyclic, and simple, therefore more recent research breaks these constraints by being bidirectional~\cite{liu_mask_2022, wang_logical_2022} or being cyclic or multi~\cite{yin_existential_2023}. To meet these new requirements, they propose to represent the formula by the query graph~\cite{yin_existential_2023}, which inherits the convention of constraint network in representing CSP instance~\cite{rossi_handbook_2006}. We utilize this design and further extend it to represent $\efok$ query that contains multiple free variables. We provide the illustration and comparison of the operator tree method and the query graph method in Figure~\ref{fig:represent_query}, where we show the strong expressiveness of the query graph method. We also provide the formal definition of query graph as follows:

\begin{definition}[Query graph]\label{def: query graph}
    Let $\gamma$ be a conjunctive formula in Equation~\ref{def:DNF formula}, its query graph is defined by $G(\gamma) = \{(h,r,t,\{\text{T, F}\})\}$, where an atomic formula $\rho=r(h,t)$ in $\gamma$ corresponds to $(h,r,t,\text{T})$ and $\rho=\lnot r(h,t)$ corresponds to $(h,r,t,\text{F})$. The \text{T, F} is short for \textbf{True, False} respectively.
\end{definition}

Therefore, any conjunctive formulas can be represented by a query graph, in the rest of the paper, we use query graphs and conjunctive formulas interchangeably.

\section{The combinatorial space of $\efok$ queries}\label{sec:nontrivial query}

\begin{figure*}[t]
\centering
\includegraphics[width=.55\linewidth]{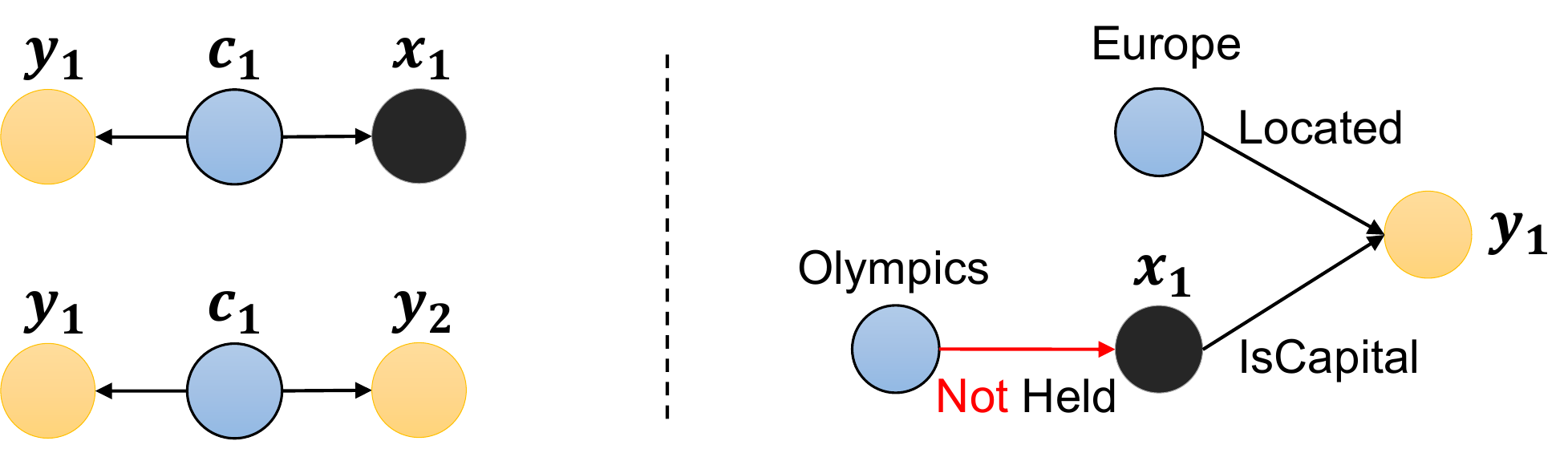}
\caption{\textbf{Left}: Example of trivial abstract query graph, in the upper left graph, the $x_1$ is redundant, violating Assumption~\ref{asm:no redundancy}, in the bottom left graph, answers for the whole query can be decomposed to answer two free variables $y_1$ and $y_2$ alone, violating Assumption~\ref{asm:no decomposition}. \textbf{Right}: Example of new query graph that is not included in previous benchmark~\cite{wang_benchmarking_2021} even though it can be represented by operator-tree. The representation of query graph follows Figure~\ref{fig:represent_query}.}
\label{fig: nontrivial query}
\vspace{-1em}
\end{figure*}

 Although previous research has given a systematic investigation in the combinatorial space of operator trees~\cite{wang_benchmarking_2021}, the combinatorial space of the query graph is much more challenging due to the extremely large search space and the lack of explicit recursive formulation. To tackle this issue on a solid theoretical background, we put forward additional assumptions to exclude trivial query graphs. Such assumptions or restrictions also exist in the previous dataset and benchmark~\cite{ren_beta_2020,wang_benchmarking_2021}. Specifically, we propose to split the task of generating data into two levels, the abstract level, where we create \emph{abstract query graph}, also known as ``query type'' in previous research~\cite{ren_query2box_2020}, and the grounded level, where we provide the abstract query graph with the relation and constant to ground it as a query graph. In this section, we elaborate on how we investigate the scope of the nontrivial $\efok$ query of \textbf{structure hardness} step by step.

\subsection{Nontrivial abstract query graph of $\efok$}

The abstract query graph is the ungrounded query graph, without information of certain knowledge graphs, and we give an example in Figure~\ref{fig: query framework}.

\begin{definition}[Abstract query graph]
    The abstract query graph $\mathcal{G}=(V,E,f,g)$ is a directed graph with three node types,$\{\textbf{Constant Entity, Existential Variable, Free variable}\}$, and two edge types,$\{\text{T, F}\}$. The $V$ is the set of nodes, $E$ is the set of directed edges, $f$ is the function maps node to node type, and $g$ is the function maps edge to edge type. We note the \text{T, F} is the same with Definition~\ref{def: query graph}.
\end{definition}

\begin{definition}[Grounding]
    For an abstract query graph $\mathcal{G}$, a grounding is a function $I$ that maps it into a query graph $I(\mathcal{G})$.
\end{definition}

We propose two assumptions of the abstract query graph as follows:

\begin{assumption}[No redundancy]\label{asm:no redundancy}
    For a abstract query graph $\mathcal{G}$, there is not a subgraph  $\mathcal{G}_s\subsetneq \mathcal{G}$ such that for every grounding $I$, $\mathcal{A}[I(\mathcal{G})]=\mathcal{A}[I(\mathcal{G}_s)]$.
\end{assumption}

\begin{assumption}[No decomposition]\label{asm:no decomposition}
    For an abstract query graph $\mathcal{G}$, there are no such two subgraphs $\mathcal{G}_{1}$, $\mathcal{G}_{2}$, satisfying that $\mathcal{G}_{1}, \mathcal{G}_{2}\subsetneq \mathcal{G}$, such that for every instantiation $I$, $\mathcal{A}[I(\mathcal{G})] = \mathcal{A}[I(\mathcal{G}_{1})] \bigtimes \mathcal{A}[I(\mathcal{G}_{2})]$, where the $\bigtimes$ represents the Cartesian product.
\end{assumption}

We note that the assumption~\ref{asm:no decomposition} inherits the idea of the \textbf{structural} decomposition technique in CSP~\cite{gottlob_comparison_2000}, which allows for solving a CSP instance by solving several sub-problems and combining the answer together based on topology property. Additionally, meeting these two assumptions in the grounded query graph is extremely computationally costly which we aim to avoid in practice.

We provide some easy examples to be excluded for violating the assumptions above in Figure~\ref{fig: nontrivial query}.

\subsection{Nontrivial query graph of $\efok$}
Similarly, we propose two assumptions on the query graph after grounding the abstract query graph $G=I(\mathcal{G})$:

\begin{assumption}[Meaningful negation]\label{asm:meaningful negation}
    For any negative edge $e$ in query graph $G$, we require removing it results in different CSP answers: $\mathcal{\overline{A}}[G-e] \neq \mathcal{\overline{A}}[G]$.\footnote{Ideally, we should expect them to have different answers as the existential formulas, however, this is computation costly and difficult to sample in practice, which is further discussed in Appendix~\ref{app:sample negative query}.}
\end{assumption}

 Assumption~\ref{asm:meaningful negation} treats negation separately because of the fact that for any $\mathcal{KG}$, any relation $r\in \relation$, there is $|\{(h,t)|h,t\in \entity, (h,r,t)\in \mathcal{KG} \}| \ll \entity^2$, which means that the constraint induced by the negation of an atomic formula is much less ``strict'' than the one induced by a positive atomic formula.

\begin{assumption}[Appropriate answer size]\label{asm:appropriate answer size}
    There is a constant $M \ll \entity$ to bound the candidate set for each free variable $f_i$ in $G$, such that for any $i$,  $|\{a_i\in \entity| ~(a_1,\cdots,a_i,\cdots,a_k)\in \mathcal{A}[G]\}| \leq M$.
\end{assumption}

We note the Assumption~\ref{asm:appropriate answer size} \textbf{extends} the ``bounded negation'' assumption in the previous dataset~\cite{ren_beta_2020,wang_benchmarking_2021}.  We give an example ``Find a city that is located in Europe and is the capital of a country that has not held the Olympics'' in Figure~\ref{fig: nontrivial query}, where the choice of variable $x_1$ is in fact bounded by its relation with the $y_1$ variable but not from the bottom ``Olympics'' constant, hence, this query is excluded in their dataset due to the directionality of operator tree.

Overall, the scope of the formula investigated in this paper surpasses the previous EFO-1-QA  benchmark because of: (1). We include the $\efok$ formula with multiple free variables for the first time; (2). We include the whole family of $\textsc{EFO}_1$ query, many of them can not be represented by operator tree; (3) Our assumption is more systematic than previous ones as shown by the example in Figure~\ref{fig: nontrivial query}. More details are offered in Appendix~\ref{app:extend benchmark}.

\section{Framework}\label{sec: Framework}

\begin{figure*}[t]
\centering
\includegraphics[width=.9\linewidth]{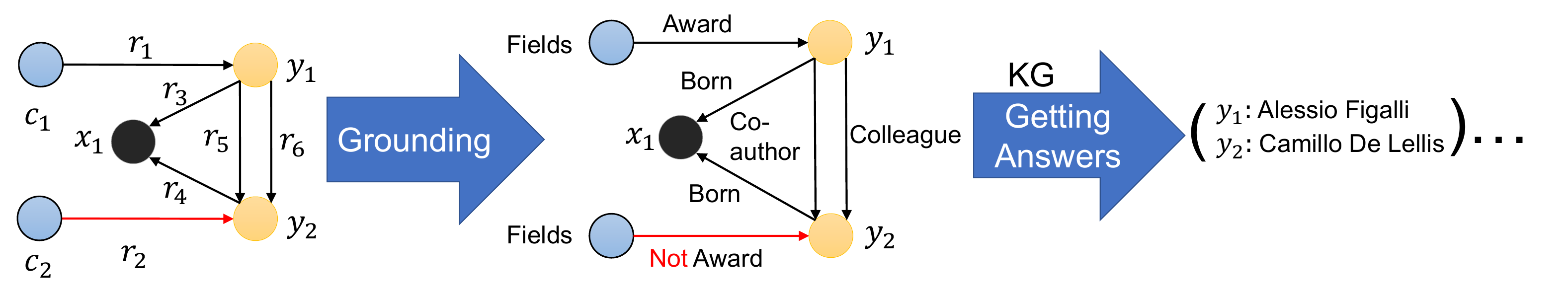}
\caption{Illustration of the functionality of our framework.\textbf{Left}: abstract query graph, \textbf{Middle}: query graph, \textbf{Right}: answer of query. The representation of query graph follows Figure~\ref{fig:represent_query}.}
\label{fig: query framework}
\vspace{-1em}
\end{figure*}
We develop a versatile framework that supports five key functionalities fundamental to the whole CQA task: (1) Enumeration of nontrivial abstract query graphs as discussed in Section~\ref{sec:nontrivial query}; (2) Sample grounding for the abstract query graph; (3) Compute answer for any query graph efficiently; (4) Support implementation of existing CQA models; (5) Conduct evaluation including newly introduced $\efok$ queries with multiple free variables. We explain each functionality in the following. An illustration of the first three functionalities is given in Figure~\ref{fig: query framework}.

\subsection{Enumerate abstract query graph }

As discussed in Section~\ref{sec:nontrivial query}, we are able to abide by those assumptions as well as \textbf{enumerate} all possible query graphs within a given search space where certain parameters, including the number of constants, free variables, existential variables, and the number of edges are all given. Additionally, we apply the graph isomorphism algorithm to avoid duplicated query graphs being generated. More details for our generation method are provided in Appendix~\ref{app:proof create qg}.

\subsection{Ground abstract query graph}\label{sec: ground abstract query graph}

To ground an abstract query graph $\mathcal{G}$ and comply with the assumption~\ref{asm:meaningful negation}, we split the abstract query graph into two parts, the positive part and the negative part, $\mathcal{G}=\mathcal{G}_p \cup \mathcal{G}_n$.  Then the grounding process is also split into two steps: 1. Sample grounding for the positive subgraph $\mathcal{G}_p$ and compute its answer 2. Ground the $\mathcal{G}_n$ to decrease the answer got in the first step. Details in Appendix~\ref{app:sample negative query}.

Finally, to fulfill the assumption~\ref{asm:appropriate answer size}, we follow the previous practice of manually filtering out queries that have more than 100 answers~\cite{ren_beta_2020,wang_benchmarking_2021}, as we have introduced the $\efok$ queries, we slightly soften this constraint to be no more than $100 \times k$ answers. 

\subsection{Answer for conjunctive query}

 As illustrated in Section~\ref{sec: conjunctive formula as CSP}, the answer to a conjunctive query can be solved by a CSP solver, however, we also show in Definition~\ref{def: CSP answer of existential formula} that CSP requires keeping track of the existential variables and it leads to huge computation costs. Thus, we develop our own algorithm following the standard solving technique of CSP, which ensures consistency conditions in the first step, and do the backtracking to get the final answers in the second step. Finally, we select part of our sampled queries and double-check its answer with the CSP solver~\url{https://github.com/python-constraint/python-constraint}.

\subsection{Learning-based methods}

As the query graph is an extension to the operator tree regarding the express ability to existential formulas, we are able to reproduce CQA models that are initially implemented by the operator tree in our new framework. Specifically, since the operator tree is directed and acyclic, we compute the topology ordering that allows for step-by-step computation in the query graph. This algorithm is illustrated in detail in the Appendix~\ref{app:implementation detail of CQA model}. We note our implementation coincides with the original one.

Conversely, for the newly proposed models that are based on query graphs, the original operator tree framework is not able to implement them, while our framework is powerful enough. We have therefore clearly shown that the query graph representation is more powerful than the previous operator tree and is able to support arbitrary conjunctive formulas as explained in Section~\ref{sec: representation of formula}.

\subsection{Evaluation protocol}\label{sec: evaluation for multiple}

As we have mentioned in Section~\ref{sec:knowledge graph definition}, there is an observed knowledge graph $\mathcal{KG}_{o}$ and a full knowledge graph $\mathcal{KG}$. Thus, there is a set of observed answers $\mathcal{A}_o$ and a set of full answers $\mathcal{A}$ correspondingly. Since the goal of CQA is to tackle the challenge of OWA, it has been a common practice to evaluate CQA models by the ``hard'' answers $\mathcal{A}_h=\mathcal{A}-\mathcal{A}_o$~\cite{ren_query2box_2020,ren_neural_2023}. However, to the best of our knowledge, there has not been a systematic evaluation protocol for $\efok$ queries, thus we leverage the idea of the previous study and propose three types of different metrics to fill the research gap in the area of evaluation of queries with multiple free variables, and thus have \textbf{combinatorial hardness}.

\noindent\textbf{Marginal.} For any free variable $f_i$, its full answer is $\mathcal{A}^{f_i}=\{a_i\in \entity| (a_1,\cdots,a_i,\cdots,a_k)\in \mathcal{A}\}$, the observed answer of it $\mathcal{A}_{o}^{f_i}$ is defined similarly. This is termed ``solution projection'' in CSP theory~\cite{greco_power_2013} to evaluate whether the locally retrieved answer can be extended to an answer for the whole problem. Then, we rank its hard answer $\mathcal{A}_{h}^{f_i} = \mathcal{A}^{f_i} - \mathcal{A}_{o}^{f_i}$\footnote{We note $\mathcal{A}_{h}^{f_i}$ can be empty for some free variables or even for all free variables, making these marginal metrics not reliable, details in Appendix~\ref{app:evaluation detail}.}, against those non-answers $\entity-\mathcal{A}^{f_i} - \mathcal{A}_{o}^{f_i}$ and use the ranking to compute standard metrics like MRR, HIT@K for every free variable. Finally, the metric on the whole query graph is taken as the average of the metric on all free variables. We note that this metric is an extension of the previous design proposed by~\cite{liu_neural-answering_2021}. However, this metric has the inherent drawback that it fails to evaluate the combinatorial answer by the $k$-length tuple and thus fails to find the correspondence among free variables.

\noindent\textbf{Multiply.} Because of the limitation of the marginal metric discussed above, we propose to evaluate the combinatorial answer by each $k$-length tuple $(a_1,\cdots,a_k)$ in the hard answer set $\mathcal{A}_{h}$. Specifically, we rank each $a_i$ in the corresponding node $f_i$ the same as the marginal metric. Then, we propose the $\text{HIT}@n^{k}$ metric, it is 1 if all $a_i$ is ranked in the top $n$ in the corresponding node $f_i$, and 0 otherwise.

\noindent\textbf{Joint.} Finally, we note these metrics above are not the standard way of evaluation, which should be based on a joint ranking for all the $\entity^k$  combinations of the entire search space. We propose to estimate the joint ranking in a closed form given certain assumptions, see Appendix~\ref{app:evaluation detail} for the proof and details.

\section{The $\efok$-CQA dataset and benchmark results}

\subsection{The $\efok$-CQA dataset}

With the help of our framework developed in Section~\ref{sec: Framework}, we are able to develop a new dataset called $\efok$-CQA, whose combinatorial space is parameterized by the number of constants, existential and free variables, and the number of edges. $\efok$-CQA dataset includes 741 different abstract query graphs in total. The parameters and the generation process, as well as it statistics is detailed in Appendix~\ref{app:EFOX statistics}.

Then, we conduct experiments on our new $\efok$-CQA dataset with six representative CQA models including BetaE~\cite{ren_beta_2020}, LogicE ~\cite{luus_logic_2021}, and ConE~\cite{zhang_cone_2021}, which are built on the operator tree, CQD~\cite{arakelyan_complex_2021}, LMPNN~\cite{wang_logical_2023}, and FIT~\cite{yin_existential_2023} which are built on query graph. The experiments are conducted in two parts, (1). The queries with one free variable, specifically, including those that can not be represented by the operator tree; (2). The queries contain multiple free variables.

We have made some adaptations to the implementation of CQA models, allowing them to infer $\efok$ queries, full detail is offered in Appendix~\ref{app:implementation detail of CQA model}. The experiment is conducted on a standard knowledge graph FB15k-237~\cite{toutanova_observed_2015} and additional experiments on other standard knowledge graphs FB15k and NELL are presented in Appendix~\ref{app: additional experiment result}.

\subsection{Benchmark results for $k=1$}\label{sec: EFO1 result}
\begin{table}[t]
\caption{MRR scores(\%) for inferring queries with one free variable on FB15k-237. We denote $e$ as the number of existential variables and $c$ as the number of constant entities. SDAG represents the Simple Directed Acyclic Graph, Multi for multigraph, and Cyclic for the cyclic graph. AVG.($c$) and AVG.($e$) is the average score of queries with the number of constant entities / existential variables fixed.}
\label{tab: EFO1 result}
\centering
\footnotesize
\begin{tabular}{cccccccccc}
\toprule
      Model            & \diagbox{$c$}{$e$} & 0   & \multicolumn{2}{c}{1} & \multicolumn{3}{c}
      {2} & \multirow{2}{*}{AVG.($c$)} & \multirow{2}{*}{AVG.} \\
      \cmidrule(lr){3-3} \cmidrule(lr){4-5} \cmidrule(lr){6-8}
                  &   & SDAG & SDAG      & Multi      & SDAG  & Multi & Cyclic &        &     \\
\midrule                  
\multirow{4}{*}{BetaE} & 1 &16.2&17.9&10.9&10.6&8.5&16.5&11.1&\multirow{3}{*}{20.7}   \\
                  & 2 &35.6&20.2&19.1&15.7&15.7&27.1&17.8    \\
                  & 3 &53.3&32.4&33.1&21.7&21.6&37.4&24.8  \\\cmidrule(lr){2-8}
                  & AVG.($e$)&37.4&25.7&23.5&18.8&18.1&30.5&\\
\midrule                  
\multirow{4}{*}{LogicE} & 1 &17.4&19.0&11.5&11.0&8.5&16.8&11.5&\multirow{3}{*}{21.3}   \\
                  & 2 &36.7&21.2&19.8&16.5&16.1&27.3&18.4    \\
                  & 3 &55.5&34.6&34.5&22.3&22.0&37.5&25.4  \\  \cmidrule(lr){2-8}
                  & AVG.($e$)&38.9&27.3&24.5&19.4&18.5&30.6&\\
\midrule                  
\multirow{4}{*}{ConE} & 1 &18.6&19.9&11.8&11.4&9.3&18.7&12.3&\multirow{4}{*}{23.1}   \\
                  & 2 &39.1&22.4&20.8&18.1&17.6&30.7&20.1    \\
                  & 3 &58.8&36.4&37.0&24.6&23.8&41.7&27.6  \\ \cmidrule(lr){2-8}
                  & AVG.($e$)&41.4&28.7&26.0&21.3&20.1&34.2&\\
\midrule        
\multirow{4}{*}{CQD} 
& 1 &\textbf{22.2}&19.5&9.0&9.2&6.4&15.6&10.0&\multirow{4}{*}{21.9}   \\
& 2 &35.3&20.1&19.1&16.4&16.2&27.6&18.4    \\
& 3 &40.3&32.9&34.3&24.4&24.0&40.2&26.8  \\ \cmidrule(lr){2-8}
& AVG.($e$)&33.9&26.2&23.7&20.5&19.4&31.9& \\
\midrule
\multirow{4}{*}{LMPNN} & 1 &20.5&21.4&11.2&11.6&8.7&17.0&11.9&\multirow{4}{*}{20.5}   \\
                  & 2 &42.0&22.6&18.5&16.5&14.9&26.5&17.9    \\
                  & 3 &62.3&35.9&31.6&22.1&19.8&35.5&24.0  \\ \cmidrule(lr){2-8}
                  & AVG.($e$)&44.2&28.8&22.7&19.4&16.9&29.4& \\
\midrule
\multirow{4}{*}{FIT} & 1 &\textbf{22.2}&\textbf{25.0}&\textbf{17.4}&\textbf{13.9}&\textbf{11.7}&\textbf{23.3}&\textbf{15.6}&\multirow{4}{*}{\textbf{30.3}}   \\
                  & 2 &\textbf{45.3}&\textbf{29.6}&\textbf{28.5}&\textbf{23.8}&\textbf{24.3}&\textbf{35.5}&\textbf{26.5}    \\
                  & 3 &\textbf{64.5}&\textbf{44.8}&\textbf{45.4}&\textbf{33.3}&\textbf{33.5}&\textbf{44.4}&\textbf{36.2}  \\ \cmidrule(lr){2-8}
                  & AVG.($e$)&\textbf{46.7}&\textbf{36.2}&\textbf{33.6}&\textbf{28.6}&\textbf{27.9}&\textbf{37.9}& \\
\bottomrule
\end{tabular}
\vspace{-1em}
\end{table}
Because of the great number of abstract query graphs, we follow~\cite{wang_benchmarking_2021} to group abstract query graphs(query type) by three factors: (1). The number of constant entities; (2). The number of existential variables, and (3). The topology of the abstract query graph\footnote{We make a further constraint in our $\efok$-CQA dataset that the total edge is at most as many as the number of nodes, thus, a graph can not be both a multigraph and a cyclic graph.}. The result is shown in Table~\ref{tab: EFO1 result}.

\noindent\textbf{Structure analysis.} Firstly, we find a clear monotonic trend that adding constant entities makes a query easier while adding existing variables makes a query harder, which is intuitively correct while the previous research~\cite{wang_benchmarking_2021} fails to uncover because of the hindrance of logical operators in the operator tree. Besides, we are the first to consider the topology of query graphs: when the number of constants and existential variables is fixed, we have found the originally investigated queries that correspond to Simple Directed Acyclic Graphs (SDAG) are generally easier than the multigraphs ones but harder than the cyclic graph ones. This is an intriguing result that greatly deviates from traditional CSP theory which finds that the cyclic graph is NP-complete, while the acyclic one is tractable~\cite{carbonnel_tractability_2016}. Our conjuncture is that the cyclic graph contains one more constraint than SDAG that serves as a source of information for CQA models, while the multigraph tightens an existing constraint and thus makes the query harder.

\noindent\textbf{Model analysis.} For models that are built on operator tree, including BetaE, LogicE, and ConE, their relative performance is steady among all breakdowns and is consistent with their reported score in the original dataset~\cite{ren_beta_2020}, showing similar generalizability. However, for models that are built on query graphs, including CQD, LMPNN, and FIT,  we have found that LMPNN performs generally better than CQD in SDAG, but falls behind CQD in multigraphs and cyclic graphs. We assume the reason behind this is that LMPNN requires training while CQD does not, however, the original dataset are \textbf{biased} which only considers SDAG, leading to the result that LMPNN doesn't generalize well to the unseen tasks with different topology property. We expect future CQA models may use our framework to address this issue of biased data and generalize better to more complex queries.

We note FIT is designed to infer all $\textsc{EFO}_1$ queries and is indeed able to outperform other models in all breakdowns, however, its performance comes with the price of computational cost, and face challenges in cyclic graph where it degenerates to enumeration: which we further explain in Appendix~\ref{app:implementation detail of CQA model}. 

\subsection{Benchmark results for $k=2$}\label{sec: EFO2 result}

\begin{table}[t]
\centering
\caption{HIT@10 scores(\%) of three different types for answering queries with two variables on FB15k-237. The constant number is fixed to be two. $e$ is the number of existential variables. The SDAG, Multi, and Cyclic are the same as Table~\ref{tab: EFO1 result}.}
\label{tab: EFO2 result}
\footnotesize
\begin{tabular}{ccrrrrrrrrr}
\toprule
\multirow{2}{*}{Model}  & \multirow{2}{*}{\shortstack[c]{HIT@10\\ Type}}  & \multicolumn{2}{c}{$e=0$} & \multicolumn{3}{c}{$e=1$} & \multicolumn{3}{c}{$e=2$} & \multirow{2}{*}{AVG.} \\
\cmidrule(lr){3-4} \cmidrule(lr){5-7} \cmidrule(lr){8-10} & & SDAG      & Multi &  SDAG  & Multi & Cyclic &  SDAG  & Multi & Cyclic & \\
\midrule
\multirow{3}{*}{BetaE} 
& Marginal& 54.5&50.2&49.5&46.0&58.8&37.2&35.5&58.3&43.8 \\
& Multiply&27.3&22.4&22.3&16.9&26.2&16.9&13.9&25.7&18.3\\
& Joint&6.3&5.4&5.2&4.2&10.8&2.2&2.3&9.5&4.5\\
\midrule
\multirow{3}{*}{LogicE} 
& Marginal&58.2&50.9&52.2&47.4&60.4&37.7&35.8&59.2&44.6 \\
& Multiply&32.1&23.1&24.9&18.1&28.3&18.1&14.8&26.6&19.5\\
& Joint&6.8&6.0&6.1&4.5&12.3&2.5&2.7&10.3&5.1 \\
\midrule
\multirow{3}{*}{ConE} 
& Marginal&60.3&53.8&54.2&50.3&\textbf{66.2}&40.1&38.5&\textbf{63.7}&47.7 \\
& Multiply&33.7&25.2&26.1&19.8&32.1&19.5&16.3&30.3&21.5\\
& Joint&6.7&6.4&6.2&4.8&12.6&2.6&2.7&10.9&5.3 \\
\midrule
\multirow{3}{*}{CQD} 
& Marginal&50.4&46.5&49.1&45.6&59.7&33.5&33.1&61.5&42.8 \\
& Multiply  &28.9&23.4&25.4&19.5&31.3&17.8&16.0&30.5&21.0\\
& Joint   &\textbf{8.0}&8.0&7.4&6.0&\textbf{13.9}&3.6&3.9&\textbf{12.0}&\textbf{6.4}  \\
\midrule
\multirow{3}{*}{LMPNN} 
& Marginal&58.4&51.1&54.9&49.2&64.7&39.6&36.1&58.7&45.4 \\
& Multiply   &35.0&26.7&29.2&21.7&\textbf{33.4}&21.4&17.0&28.4&22.2\\
& Joint   &7.6&7.5&7.1&5.3&12.9&2.8&2.9&9.5&5.2  \\
\midrule
\multirow{3}{*}{FIT} 
& Marginal&\textbf{64.3}&\textbf{61.0}&\textbf{63.1}&\textbf{60.7}&58.5&\textbf{49.0}&\textbf{49.1}&60.2&\textbf{54.3} \\
& Multiply  &\textbf{39.7}&\textbf{32.2}&\textbf{35.9}&\textbf{27.8}&27.4&\textbf{29.5}&\textbf{26.8}&\textbf{32.4}&\textbf{29.2}\\
& Joint    &7.4&\textbf{9.0}&\textbf{7.8}&\textbf{6.5}&10.1&\textbf{3.7}&\textbf{4.6}&10.6&\textbf{6.4}   \\
\bottomrule
\end{tabular}
\vspace{-1em}
\end{table}
As we have explained in Section~\ref{sec: evaluation for multiple}, we propose three kinds of metrics, marginal ones, multiply ones, and joint ones, from easy to hard, to evaluate the performance of a model in the scenario of multiple variables. The evaluation result is shown in Table~\ref{tab: EFO2 result}. As the effect of the number of constant variables is quite clear, we remove it and add the metrics based on $\text{HIT}@10$ as the new factor. 

For the impact regarding the number of existential variables and the topology property of the query graph, we find the result is similar to Table~\ref{tab: EFO1 result}, which may be explained by the fact that those models are all initially designed to infer queries with one free variable.
For the three metrics we have proposed, we have identified a clear difficulty difference among them though they generally show similar trends. The scores of joint HIT@10 are pretty low, indicating the great hardness of answering queries with multiple variables. Moreover, we have found that FIT falls behind other models in some breakdowns which are mostly cyclic graphs, corroborating our discussion in Section~\ref{sec: EFO1 result}. 

\section{Conclusion}
In this paper, we make a thorough investigation of the family of $\efok$ formulas based on solid theoretical background. We then present a new powerful framework that supports several functionalities essential to CQA task, with this help, we build the $\efok$-CQA dataset that greatly extends the previous dataset and benchmark. Our evaluation result brings new empirical findings and reflects the biased selection in the previous dataset impairs the performance of CQA models, emphasizing the contribution of our work.

\bibliographystyle{plainnat}
\bibliography{ref}

\begin{thebibliography}{32}
\providecommand{\natexlab}[1]{#1}
\providecommand{\url}[1]{\texttt{#1}}
\expandafter\ifx\csname urlstyle\endcsname\relax
  \providecommand{\doi}[1]{doi: #1}\else
  \providecommand{\doi}{doi: \begingroup \urlstyle{rm}\Url}\fi

\bibitem[Arakelyan et~al.(2021)Arakelyan, Daza, Minervini, and
  Cochez]{arakelyan_complex_2021}
Erik Arakelyan, Daniel Daza, Pasquale Minervini, and Michael Cochez.
\newblock Complex query answering with neural link predictors.
\newblock \emph{arXiv preprint arXiv:2011.03459}, 2021.

\bibitem[Bai et~al.(2022)Bai, Wang, Zhang, and Song]{bai_query2particles_2022}
Jiaxin Bai, Zihao Wang, Hongming Zhang, and Yangqiu Song.
\newblock {Query2Particles}: {Knowledge} {Graph} {Reasoning} with {Particle}
  {Embeddings}.
\newblock \emph{arXiv preprint arXiv:2204.12847}, 2022.

\bibitem[Bordes et~al.(2013)Bordes, Usunier, Garcia-Duran, Weston, and
  Yakhnenko]{bordes_translating_2013}
Antoine Bordes, Nicolas Usunier, Alberto Garcia-Duran, Jason Weston, and Oksana
  Yakhnenko.
\newblock Translating {Embeddings} for {Modeling} {Multi}-relational {Data}.
\newblock In \emph{Advances in {Neural} {Information} {Processing} {Systems}},
  volume~26. Curran Associates, Inc., 2013.
\newblock URL
  \url{https://papers.nips.cc/paper_files/paper/2013/hash/1cecc7a77928ca8133fa24680a88d2f9-Abstract.html}.

\bibitem[Carbonnel and Cooper(2016)]{carbonnel_tractability_2016}
Cl{\'e}ment Carbonnel and Martin~C Cooper.
\newblock Tractability in constraint satisfaction problems: a survey.
\newblock \emph{Constraints}, 21\penalty0 (2):\penalty0 115--144, 2016.
\newblock Publisher: Springer.

\bibitem[Carlson et~al.(2010)Carlson, Betteridge, Kisiel, Settles, Hruschka,
  and Mitchell]{carlson_toward_2010}
Andrew Carlson, Justin Betteridge, Bryan Kisiel, Burr Settles, Estevam
  Hruschka, and Tom Mitchell.
\newblock Toward an architecture for never-ending language learning.
\newblock In \emph{Proceedings of the {AAAI} conference on artificial
  intelligence}, volume~24, pages 1306--1313, 2010.
\newblock Issue: 1.

\bibitem[Ehrlinger and W{\"o}{\ss}(2016)]{ehrlinger_towards_2016}
Lisa Ehrlinger and Wolfram W{\"o}{\ss}.
\newblock Towards a definition of knowledge graphs.
\newblock \emph{SEMANTiCS (Posters, Demos, SuCCESS)}, 48\penalty0
  (1-4):\penalty0 2, 2016.

\bibitem[Gottlob et~al.(1999)Gottlob, Leone, and
  Scarcello]{gottlob_hypertree_1999}
Georg Gottlob, Nicola Leone, and Francesco Scarcello.
\newblock Hypertree decompositions and tractable queries.
\newblock In \emph{Proceedings of the eighteenth {ACM}
  {SIGMOD}-{SIGACT}-{SIGART} symposium on {Principles} of database systems},
  pages 21--32, 1999.

\bibitem[Gottlob et~al.(2000)Gottlob, Leone, and
  Scarcello]{gottlob_comparison_2000}
Georg Gottlob, Nicola Leone, and Francesco Scarcello.
\newblock A comparison of structural {CSP} decomposition methods.
\newblock \emph{Artificial Intelligence}, 124\penalty0 (2):\penalty0 243--282,
  December 2000.
\newblock ISSN 0004-3702.
\newblock \doi{10.1016/S0004-3702(00)00078-3}.
\newblock URL
  \url{https://www.sciencedirect.com/science/article/pii/S0004370200000783}.

\bibitem[Greco and Scarcello(2013)]{greco_power_2013}
Gianluigi Greco and Francesco Scarcello.
\newblock On {The} {Power} of {Tree} {Projections}: {Structural} {Tractability}
  of {Enumerating} {CSP} {Solutions}.
\newblock \emph{Constraints}, 18\penalty0 (1):\penalty0 38--74, January 2013.
\newblock ISSN 1383-7133, 1572-9354.
\newblock \doi{10.1007/s10601-012-9129-8}.
\newblock URL \url{http://arxiv.org/abs/1005.1567}.
\newblock arXiv:1005.1567 [cs].

\bibitem[Hamilton et~al.(2018)Hamilton, Bajaj, Zitnik, Jurafsky, and
  Leskovec]{hamilton_embedding_2018}
Will Hamilton, Payal Bajaj, Marinka Zitnik, Dan Jurafsky, and Jure Leskovec.
\newblock Embedding logical queries on knowledge graphs.
\newblock \emph{Advances in neural information processing systems}, 31, 2018.

\bibitem[Kolaitis and Vardi(1998)]{kolaitis_conjunctive-query_1998}
Phokion~G Kolaitis and Moshe~Y Vardi.
\newblock Conjunctive-query containment and constraint satisfaction.
\newblock In \emph{Proceedings of the seventeenth {ACM}
  {SIGACT}-{SIGMOD}-{SIGART} symposium on {Principles} of database systems},
  pages 205--213, 1998.

\bibitem[Leskovec(2023)]{leskovec_databases_2023}
Jure Leskovec.
\newblock Databases as {Graphs}: {Predictive} {Queries} for {Declarative}
  {Machine} {Learning}.
\newblock In \emph{Proceedings of the 42nd {ACM} {SIGMOD}-{SIGACT}-{SIGAI}
  {Symposium} on {Principles} of {Database} {Systems}}, {PODS} '23, page~1, New
  York, NY, USA, 2023. Association for Computing Machinery.
\newblock ISBN 9798400701276.
\newblock \doi{10.1145/3584372.3589939}.
\newblock URL \url{https://doi.org/10.1145/3584372.3589939}.
\newblock event-place: Seattle, WA, USA.

\bibitem[Libkin and Sirangelo(2009)]{libkin_open_2009}
Leonid Libkin and Cristina Sirangelo.
\newblock Open and {Closed} {World} {Assumptions} in {Data} {Exchange}.
\newblock \emph{Description Logics}, 477, 2009.

\bibitem[Liu et~al.(2021)Liu, Du, Ji, Zhai, and
  Tong]{liu_neural-answering_2021}
Lihui Liu, Boxin Du, Heng Ji, ChengXiang Zhai, and Hanghang Tong.
\newblock Neural-{Answering} {Logical} {Queries} on {Knowledge} {Graphs}.
\newblock In \emph{Proceedings of the 27th {ACM} {SIGKDD} {Conference} on
  {Knowledge} {Discovery} \& {Data} {Mining}}, pages 1087--1097, 2021.

\bibitem[Liu et~al.(2022)Liu, Zhao, Su, Cen, Qiu, Zhang, Wu, Dong, and
  Tang]{liu_mask_2022}
Xiao Liu, Shiyu Zhao, Kai Su, Yukuo Cen, Jiezhong Qiu, Mengdi Zhang, Wei Wu,
  Yuxiao Dong, and Jie Tang.
\newblock Mask and {Reason}: {Pre}-{Training} {Knowledge} {Graph}
  {Transformers} for {Complex} {Logical} {Queries}.
\newblock In \emph{Proceedings of the 28th {ACM} {SIGKDD} {Conference} on
  {Knowledge} {Discovery} and {Data} {Mining}}, pages 1120--1130, August 2022.
\newblock \doi{10.1145/3534678.3539472}.
\newblock URL \url{http://arxiv.org/abs/2208.07638}.
\newblock arXiv:2208.07638 [cs].

\bibitem[Luus et~al.(2021)Luus, Sen, Kapanipathi, Riegel, Makondo, Lebese, and
  Gray]{luus_logic_2021}
Francois Luus, Prithviraj Sen, Pavan Kapanipathi, Ryan Riegel, Ndivhuwo
  Makondo, Thabang Lebese, and Alexander Gray.
\newblock Logic embeddings for complex query answering.
\newblock \emph{arXiv preprint arXiv:2103.00418}, 2021.

\bibitem[Poess and Floyd(2000)]{poess_new_2000}
Meikel Poess and Chris Floyd.
\newblock New {TPC} benchmarks for decision support and web commerce.
\newblock \emph{ACM Sigmod Record}, 29\penalty0 (4):\penalty0 64--71, 2000.
\newblock Publisher: ACM New York, NY, USA.

\bibitem[Ren et~al.(2020)Ren, Hu, and Leskovec]{ren_query2box_2020}
H~Ren, W~Hu, and J~Leskovec.
\newblock Query2box: {Reasoning} {Over} {Knowledge} {Graphs} {In} {Vector}
  {Space} {Using} {Box} {Embeddings}.
\newblock In \emph{International {Conference} on {Learning} {Representations}
  ({ICLR})}, 2020.

\bibitem[Ren and Leskovec(2020)]{ren_beta_2020}
Hongyu Ren and Jure Leskovec.
\newblock Beta embeddings for multi-hop logical reasoning in knowledge graphs.
\newblock \emph{Advances in Neural Information Processing Systems},
  33:\penalty0 19716--19726, 2020.

\bibitem[Ren et~al.(2023)Ren, Galkin, Cochez, Zhu, and
  Leskovec]{ren_neural_2023}
Hongyu Ren, Mikhail Galkin, Michael Cochez, Zhaocheng Zhu, and Jure Leskovec.
\newblock Neural {Graph} {Reasoning}: {Complex} {Logical} {Query} {Answering}
  {Meets} {Graph} {Databases}, March 2023.
\newblock URL \url{http://arxiv.org/abs/2303.14617}.
\newblock arXiv:2303.14617 [cs].

\bibitem[Rossi et~al.(2006)Rossi, van Beek, and Walsh]{rossi_handbook_2006}
Francesca Rossi, Peter van Beek, and Toby Walsh.
\newblock \emph{Handbook of {Constraint} {Programming}}.
\newblock Elsevier Science Inc., USA, 2006.
\newblock ISBN 978-0-08-046380-3.

\bibitem[Suchanek et~al.(2007)Suchanek, Kasneci, and
  Weikum]{suchanek_yago_2007}
Fabian~M Suchanek, Gjergji Kasneci, and Gerhard Weikum.
\newblock Yago: a core of semantic knowledge.
\newblock In \emph{Proceedings of the 16th international conference on {World}
  {Wide} {Web}}, pages 697--706, 2007.

\bibitem[Toutanova and Chen(2015)]{toutanova_observed_2015}
Kristina Toutanova and Danqi Chen.
\newblock Observed versus latent features for knowledge base and text
  inference.
\newblock In \emph{Proceedings of the 3rd workshop on continuous vector space
  models and their compositionality}, pages 57--66, 2015.

\bibitem[Vrande{\v c}i{\'c} and Kr{\"o}tzsch(2014)]{vrandecic_wikidata_2014}
Denny Vrande{\v c}i{\'c} and Markus Kr{\"o}tzsch.
\newblock Wikidata: a free collaborative knowledgebase.
\newblock \emph{Communications of the ACM}, 57\penalty0 (10):\penalty0 78--85,
  2014.
\newblock Publisher: ACM New York, NY, USA.

\bibitem[Wang et~al.(2021)Wang, Yin, and Song]{wang_benchmarking_2021}
Zihao Wang, Hang Yin, and Yangqiu Song.
\newblock Benchmarking the {Combinatorial} {Generalizability} of {Complex}
  {Query} {Answering} on {Knowledge} {Graphs}.
\newblock \emph{Proceedings of the Neural Information Processing Systems Track
  on Datasets and Benchmarks}, 1, December 2021.
\newblock URL
  \url{https://datasets-benchmarks-proceedings.neurips.cc/paper/2021/hash/7eabe3a1649ffa2b3ff8c02ebfd5659f-Abstract-round2.html}.

\bibitem[Wang et~al.(2022)Wang, Yin, and Song]{wang_logical_2022}
Zihao Wang, Hang Yin, and Yangqiu Song.
\newblock Logical {Queries} on {Knowledge} {Graphs}: {Emerging} {Interface} of
  {Incomplete} {Relational} {Data}.
\newblock \emph{Data Engineering}, page~3, 2022.

\bibitem[Wang et~al.(2023{\natexlab{a}})Wang, Fei, Yin, Song, Wong, and
  See]{wang_wasserstein-fisher-rao_2023}
Zihao Wang, Weizhi Fei, Hang Yin, Yangqiu Song, Ginny~Y Wong, and Simon See.
\newblock Wasserstein-{Fisher}-{Rao} {Embedding}: {Logical} {Query}
  {Embeddings} with {Local} {Comparison} and {Global} {Transport}.
\newblock \emph{arXiv preprint arXiv:2305.04034}, 2023{\natexlab{a}}.

\bibitem[Wang et~al.(2023{\natexlab{b}})Wang, Song, Wong, and
  See]{wang_logical_2023}
Zihao Wang, Yangqiu Song, Ginny Wong, and Simon See.
\newblock Logical {Message} {Passing} {Networks} with {One}-hop {Inference} on
  {Atomic} {Formulas}.
\newblock In \emph{The {Eleventh} {International} {Conference} on {Learning}
  {Representations}}, 2023{\natexlab{b}}.
\newblock URL \url{https://openreview.net/forum?id=SoyOsp7i_l}.

\bibitem[Xu et~al.(2022)Xu, Zhang, Ye, Chen, and Chen]{xu_neural-symbolic_2022}
Zezhong Xu, Wen Zhang, Peng Ye, Hui Chen, and Huajun Chen.
\newblock Neural-{Symbolic} {Entangled} {Framework} for {Complex} {Query}
  {Answering}, September 2022.
\newblock URL \url{http://arxiv.org/abs/2209.08779}.
\newblock arXiv:2209.08779 [cs].

\bibitem[Yin et~al.(2023)Yin, Wang, and Song]{yin_existential_2023}
Hang Yin, Zihao Wang, and Yangqiu Song.
\newblock On {Existential} {First} {Order} {Queries} {Inference} on {Knowledge}
  {Graphs}, April 2023.
\newblock URL \url{http://arxiv.org/abs/2304.07063}.
\newblock arXiv:2304.07063 [cs].

\bibitem[Zhang et~al.(2021)Zhang, Wang, Chen, Ji, and Wu]{zhang_cone_2021}
Zhanqiu Zhang, Jie Wang, Jiajun Chen, Shuiwang Ji, and Feng Wu.
\newblock Cone: {Cone} embeddings for multi-hop reasoning over knowledge
  graphs.
\newblock \emph{Advances in Neural Information Processing Systems},
  34:\penalty0 19172--19183, 2021.

\bibitem[Zhou et~al.(2007)Zhou, Ren, Medo, and Zhang]{zhou_bipartite_2007}
Tao Zhou, Jie Ren, Mat{\'u}{\v s} Medo, and Yi-Cheng Zhang.
\newblock Bipartite network projection and personal recommendation.
\newblock \emph{Physical review E}, 76\penalty0 (4):\penalty0 046115, 2007.
\newblock Publisher: APS.

\end{thebibliography}

\clearpage

\appendix

\section{Related works}\label{app: related works}


Answering complex queries on knowledge graphs differs from database query answering by being a data-driven task~\cite{wang_logical_2022}, where the open-world assumption is addressed by methods that learn from data. Meanwhile, learning-based methods enable faster neural approximate solutions of symbolic query answering problems~\cite{ren_neural_2023}.

The prevailing way is query embedding, where the computational results are embedded and computed in the low-dimensional embedding space. Specifically, the query embedding over the set operator trees is the earliest proposed~\cite{hamilton_embedding_2018}. The supported set operators include projection\cite{hamilton_embedding_2018}, intersection~\cite{ren_query2box_2020}, union, and negation~\cite{ren_beta_2020}, and later on be improved by various designs~\cite{xu_neural-symbolic_2022,bai_query2particles_2022,wang_wasserstein-fisher-rao_2023}. Such methods assume queries can be converted into the recursive execution of set operations, which imposes additional assumptions on the solvable class of queries~\cite{wang_benchmarking_2021}. These assumptions introduce additional limitations of such query embeddings

Recent advancements in CQA models surpass the query embedding methods by adopting query graph representation and graph neural networks, supporting atomic formulas~\cite{liu_mask_2022} and negated atomic formulas~\cite{wang_logical_2023}. Query embedding on graphs bypasses the assumptions for queries~\cite{wang_benchmarking_2021}. Meanwhile, other search-based inference methods~\cite{arakelyan_complex_2021,yin_existential_2023} are rooted in fuzzy calculus and not subject to the query assumptions~\cite {wang_benchmarking_2021}.

Though many efforts have been made, the datasets of complex query answering are usually subject to the assumptions by set operator query embeddings~\cite{wang_benchmarking_2021}. Many other datasets are proposed to enable queries with additional features, see~\cite{ren_neural_2023} for a comprehensive survey of datasets. However, only one small dataset proposed by~\cite{yin_existential_2023} introduced queries and answers beyond such assumptions~\cite{wang_benchmarking_2021}. It is questionable that this small dataset is fair enough to justify the advantages claimed in advancement methods~\cite{wang_logical_2023,yin_existential_2023} that aim at complex query answering. Moreover, query with multiple free variables has not been investigated. Therefore, the dataset~\cite{yin_existential_2023} is still far away from the systematical evaluation as~\cite{wang_benchmarking_2021} and $\efok$-CQA proposed in this paper fills this gap.

\section{Details of constraint satisfaction problem}\label{app:detail of CSP}

In this section, we introduce the constraint satisfaction problem (CSP) again. One instance of CSP $\mathcal{P}$ can be represented by a triple $\mathcal{P}=(X,D,C)$ where $X=(x_1,\cdots, x_n)$ is an $n$-tuple of variables, $D=(D_1,\cdots, D_n)$ is the corresponding $n$-tuple of domains for each variable $x_i$. Then, $C=(C_1,\cdots, C_t)$ is $t$-tuple of constraints, each constraint $C_i$ is a pair of $(S_i, R_{S_i})$ where $S_i$ is called the scope of the constraint, contains $k$ corresponding variables, which means it is a set of variables $S_i=\{x_{i_j}\}_{j=1}^{k}$ and $R_{{S_i}}$ is the constraint over those variables~\cite{rossi_handbook_2006}, meaning that $R_{S_i}$ is a subset of the cartesian product of variables in $S_i$.

The answer of the CSP instance $\mathcal{P}$ is a $n$-tuple $(a_1,\cdots,a_n)$ which is essentially an assignment for all variables $x_i$, such that: 
\begin{align*}
    a_i\in D_i,~\forall i=1,\cdots,n & \\
    (a_{i_1}, \cdots, a_{i_k} ) \in R_{S_i}    ,~\forall i=1,\cdots,t
\end{align*}

Then the formulation of existential conjunctive formulas as CSP has already been discussed in Section~\ref{sec: conjunctive formula as CSP}. Additionally, for the negation of atomic formula $\lnot r(h,t)$, we note the constraint $C$ is also binary with $S_i = \{h,t\}$, $R_{S_i}=\{(h,t) | h,t\in \entity, (h,r,t)\notin \mathcal{KG}\}$, this means that $R_{S_i}$ is a very large set, thus the constraint is less ``strict'' than the positive ones, explaining why we treat negation separately in Assumption~\ref{asm:meaningful negation}.

\section{Construction of the whole $\efok$-CQA datset}\label{app:construction of efok}

In this section, we provide details for the construction of the $\efok$-CQA dataset.
 
\subsection{Enumeration of the abstract query graphs}\label{app:proof create qg}

\begin{figure*}[t]
\centering
\includegraphics[width=.9\linewidth]{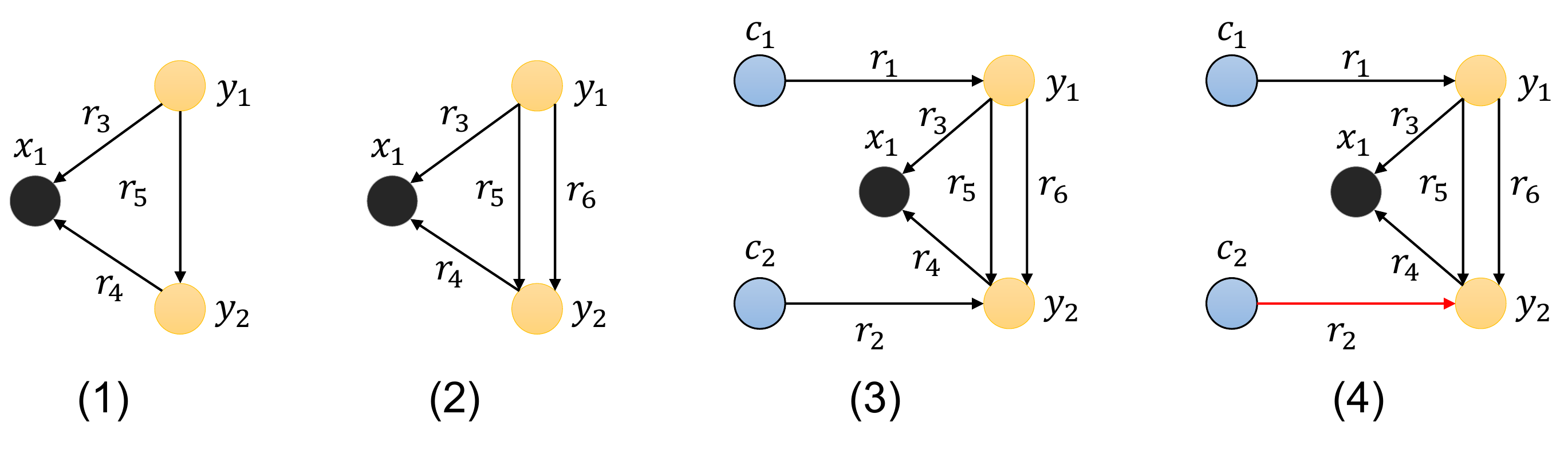}
\caption{The four steps of enumerating the abstract query graphs. We note that the example and representation follow Figure~\ref{fig: query framework}.}
\label{fig: abstract query graph steps}
\end{figure*}

To investigate the whole combinatorial space for the $\efok$ queries, we need to give propositions of the property of abstract query graph based on the assumptions given in~\ref{sec:nontrivial query}:

\begin{proposition}\label{prop: remove constant}
    For an abstract query graph $\mathcal{G}$, if it conforms Assumption~\ref{asm:no redundancy} and Assumption~\ref{asm:no decomposition}, then removing all constant entities in $\mathcal{G}$ will lead to only one connected component and no edge is connected between two constant entities.
\end{proposition}
\begin{proof}
      We prove this by contradiction. If there is an edge (whether positive or negative) between constant entities, then this edge is redundant, violating Assumption~\ref{asm:no redundancy}. Then, if there is more than one connected component after removing all constant entities in $\mathcal{G}$. Suppose one connected component has no free variable, then this part  is a sentence and thus has a certain truth value, whether 0 or 1, which is redundant, violating Assumption~\ref{asm:no redundancy}. Then, we assume every connected component has at least one free variable, we assume there is $m$ connected component and we have:
    \begin{equation*}
        Node(\mathcal{G})=(\cup_{i=1}^{m} Node(\mathcal{G}_{i})) \cup Node(\mathcal{G}_c)
    \end{equation*}
    where $m > 1$, the $\mathcal{G}_c$ is the set of constant entities and each $\mathcal{G}_i$ is the graph of a connected component, we use $Node(\mathcal{G})$ to denote the node set for a graph $\mathcal{G}$. Then this equation describes the partition of the node set of the original $\mathcal{G}$.

    Then, we construct $\mathcal{G}_a = G[Node(\mathcal{G}_1)\cup \mathcal{G}_c]$ and $\mathcal{G}_b = G[(\cup_{i=2}^{m}Node(\mathcal{G}_{i}))\cup Node(\mathcal{G}_c)]$, where $G$ represents the induced graph. Then we naturally have that $\mathcal{A}[I(\mathcal{G})] = \mathcal{A}[I(\mathcal{G}_{a})] \bigtimes \mathcal{A}[I(\mathcal{G}_{b})]$, where the $\bigtimes$ represents the Cartesian product, violating Assumption~\ref{asm:no decomposition}.
    


    
\end{proof}

Additionally, as mentioned in Appendix~\ref{app:detail of CSP}, the negative constraint is less ``strict'', we formally put an additional assumption of the real knowledge graph as the following:

\begin{assumption}\label{asm:sparse knowledge graph}
    For any knowledge graph $\mathcal{KG}$, with its entity set $\entity$ and relation set $\relation$, we assume it is sparse with regard to each relation, meaning: for any $r\in \relation, |\{a\in \entity| \exists b.~(a,r,b)\in \mathcal{KG} \text{ or } (b,r,a)\in \mathcal{KG}\}| \ll \entity $ 
\end{assumption}

Then we develop another proposition for the abstract query graph:

\begin{proposition}\label{prop: new bounded negation}
    With the knowledge graph conforming Assumption~\ref{asm:sparse knowledge graph}, for any node $u$ in the abstract query graph $\mathcal{G}$, if $u$ is an existential variable or free variable, then it can not only connect with negative edges.
\end{proposition}

\begin{proof}

     Suppose $u$ only connects to $m$ negative edge $e_1,\cdots,e_m$. For any grounding $I$, we assume $I(e_i)=r_i\in \relation$. For each $r_i$, we construct its endpoint set $$\text{Endpoint}(r_i)=\{a\in \entity| \exists b. (a,r,b)\in \mathcal{KG} \text{ or } (b,r,a)\in \mathcal{KG}\}$$ by the assumption~\ref{asm:sparse knowledge graph}, we have $|\text{Endpoint}(r_i)| \ll \entity$, then we have:  $$|\cup_{i=1}^{m}  \text{Endpoint}(r_i)| \leq \Sigma_{i=1}^{m} |\text{Endpoint}(r_i)| \ll \entity$$
    since $m$ is small due to the size of the abstract query graph. Then we have two situations about the type of node $u$:

    \noindent\textbf{1.If node $u$ is an existential variable.} 

    Then we construct a subgraph $\mathcal{G}_s$ be the induced subgraph of $Node(\mathcal{G})-u$, then for any possible grounding $I$, we prove that $\mathcal{A}[I(\mathcal{G}_s)]$=$\mathcal{A}[I(\mathcal{G})]$, the right is clearly a subset of the left due to it contains more constraints, then we show every answer of the left is also an answer on the right, we merely need to give an appropriate assignment in the entity set for node $v$, and in fact, we choose any entity in the set $\entity - \cup_{i=1}^{m}  \text{Endpoint}(r_i)$ since it suffices to satisfies all constraints of node $u$, and we have proved that $|\entity - \cup_{i=1}^{m}  \text{Endpoint}(r_i)| > 0$.

    This violates the Assumption~\ref{asm:no redundancy} as node $u$ is redundant.

    \noindent\textbf{2.If node $u$ is a free variable.} 

    Similarly, any entity in the set $\entity - \cup_{i=1}^{m}  \text{Endpoint}(r_i)$ will be an answer for the node $u$, thus violating the Assumption~\ref{asm:appropriate answer size}.
    
\end{proof}

We note the proposition~\ref{prop: new bounded negation} extends the previous requirement about negative queries, which is firstly proposed in \cite{ren_beta_2020}, inherited and termed as ``bounded negation'' in~\cite{wang_benchmarking_2021}, the ``bounded negation'' requires the negation operator should be followed by the intersection operator in the operator tree. Obviously, the abstract query graph that conforms to ``bounded negation'' will also conform to the requirement in Proposition~\ref{prop: new bounded negation}, however, an abstract query graph may abide by Proposition~\ref{prop: new bounded negation} but violates ``bounded negation'' and still represents meaningful queries. A vivid example is offered in Figure~\ref{fig: nontrivial query}, showing that our propositions about the abstract query graph are more satisfactory.

Finally, we make a more detailed assumption of diameters of the query graph:

\begin{assumption}[Appropriate diameter]
    There is a constant number $d$, such that for every node $u$ in the abstract query graph $\mathcal{G}$, it can find a free variable in its $d$-hop neighbor.
\end{assumption}

We have this assumption to exclude the extremely long-path queries that have large graph diameters.


Equipped with the propositions and assumptions above, we explore the combinatorial space of the abstract query graph given certain hyperparameters, including: the max number of free variables, max number of existential variables, max number of constant entities, max number of all nodes, max number of all edges, max number of edges surpassing the number of nodes, max number of negative edges, max distance to the free variable. In practice, these numbers are set to be: 2, 2, 3, 6, 6, 0, 1, 3. We note that the max number of edges surpassing the number of nodes is set to 0, which means that the query graph can at most have one more edge than a simple tree, thus, we exclude those query graphs that are both cyclic graphs and multigraphs, making our categorization and discussion in the experiments in Section~\ref{sec: EFO1 result} and Section~\ref{sec: EFO2 result} much more straightforward and clear.

Then, we create the abstract query graph by the following steps, which is a graph with three types of nodes and two kinds of edges:

\begin{enumerate}
    \item First, create a simple connected graph  $\mathcal{G}_1$ with two types of nodes, the existential variable and the free variable, and one type of edge, the positive edge.
    \item We add additional edges to the simple graph $\mathcal{G}_1$ and make it a multigraph $\mathcal{G}_2$.
    \item Then, the constant variable is added to the graph $\mathcal{G}_2$,  In this step, we make sure not too long existential leaves. The result is graph $\mathcal{G}_3$.
    \item Finally, random edges in $\mathcal{G}_3$ are replaced by the negation edge, and we get the final abstract query graph $\mathcal{G}_4$.
\end{enumerate}

In this way, all possible query graphs within a certain combinatorial space are enumerated, and finally, we filter duplicated graphs with the help of the graph isomorphism algorithm. We give an example to illustrate the four-step construction of an abstract query graph in Figure~\ref{fig: abstract query graph steps}.

\subsection{Ground abstract query graph with meaningful negation}\label{app:sample negative query}

To fulfill the Assumption~\ref{asm:meaningful negation} as discussed in Section~\ref{sec: ground abstract query graph}, for an abstract query graph $\mathcal{G}=(V,E,f,g)$, we have two steps:
(1). Sample grounding for the positive subgraph $\mathcal{G}_p$ and compute its answer (2). Ground the $\mathcal{G}_n$ to decrease the answer got in the first step.
To be specific, we define positive subgraph $\mathcal{G}_p$ to be defined as such, its edge set $E^{\prime} = \{e\in E | g(e)= positive\}$, its node set $V^{\prime}=\{u| u\in V, \exists e\in E^{\prime} \text{ and } e \text{ connects to } u \}$. Then $\mathcal{G}_p$=$(V^{\prime},E^{\prime},f,g)$. We note that because of Proposition~\ref{prop: new bounded negation}, if a node $u\in V-V^{\prime}$, then we know node $u$ must be a constant entity. 

Then we sample the grounding for the positive subgraph $\mathcal{G}_p$, we also compute the CSP answer for each variable $u_i$ in the query graph, 
 let $\mathcal{\overline{A}}_p$ be the whole answer for this subgraph, then we only need to compute for each variable $u_i$, its projected answer: $$\mathcal{\overline{A}}_p^{u_i}=\{a_o\in \entity|~\exists (a_1,\cdots,a_i,\cdots,a_k) \in \mathcal{\overline{A}}_p\}$$ where $k$ is the number of the variables.

 We note that computing $\mathcal{\overline{A}}_p$ is costly while computing every $\mathcal{\overline{A}}_p^{i}$ is efficient, as we only need to meet the consistency condition rather than do all the backtracking in the traditional CSP algorithm~\cite{rossi_handbook_2006}.

Then we ground what is left in the positive subgraph, we split each negative edge in $E-E^{\prime}$ into two categories:

\noindent\textbf{1. This edge $e$ connects two nodes $u_1,u_2$, and $u_1,u_2\in V^{\prime}$.}

In this case, we sample the relation $r$ to be the grounding of $e$ such that it negates some of the answers in $\mathcal{\overline{A}}_p$, actually, considering the set $\mathcal{\overline{A}}_p^{u_1}$ and $\mathcal{\overline{A}}_p^{u_2}$ is enough.

\noindent\textbf{2. This edge $e$ connects two nodes $u_1,u_2$, where $u_1 \in V^{\prime}$, while $u_2\notin V^{\prime}$.}

In this case, we sample the relation $r$ for $e$ and entity $a$ for $u_2$ such that they negate some answer in $\mathcal{\overline{A}}_p$, we note we only need to consider the \textbf{local} information, namely, possible candidates for node $u_i$, the set $\mathcal{\overline{A}}_p^{u_i}$ and it is quite efficient. However, computing the answer set as existential formula requires considering the \textbf{whole} query graph.

We note that there is no possibility that neither of the endpoints is in $V^{\prime}$ because as we have discussed above, this means that both nodes are constant entities, but in Proposition~\ref{prop: remove constant} we have asserted that no edge is connected between two entities.

\subsection{The comparison to previous benchmark}\label{app:extend benchmark}

As discussed in Section~\ref{sec:nontrivial query}, the scope of the formula investigated in our $\efok$-CQA dataset surpasses the previous EFO-1-QA  benchmark because of three reasons: (1). We include the $\efok$ formula with multiple free variables; (2). We include those $\textsc{EFO}_1$ queries that can not be represented by operator tree; (3) Our assumption is more systematic than previous ones as shown by the example in Figure~\ref{fig: nontrivial query}. Though we only contain 741 query types while the  EFO-1-QA  benchmark contains 301 query types, we list reasons for the number of query types is not significantly larger than the previous benchmark: (1). EFO-1-QA  benchmark relies on the operator tree that contains union, which represents the logic conjunction($\lor$), however, we only discuss the conjunctive queries because we always utilize the DNF of a query. We notice that there are only 129 query types in EFO-1-QA without the union, significantly smaller than the $\efok$-CQA dataset. (2). In the construction of $\efok$-CQA dataset, we restrict the query graph to have at most one negative edge to avoid the total number of query types growing quadratically, while in EFO-1-QA benchmark, their restrictions are different than ours and it contains queries that have two negative atomic formulas.

\subsection{$\efok$-CQA statistics}\label{app:EFOX statistics}

The statistics of our $\efok$-CQA dataset are shown in Table~\ref{tab: abstract query graph statistics of EFO1} and Table~\ref{tab: abstract query graph statistics of EFO2}, they show the statistics of our abstract query graph by their topology property, the statistics are split into the situation that the number of free variable $k=1$ and the number of free variable $k=2$, correspondingly. We note abstract query graphs with seven nodes have been excluded as the setting of hyperparameters discussed in Appendix~\ref{app:proof create qg}, we make these restrictions to control the quadratic growth in the number of abstract query graphs.

Finally, in FB15k-237, we sample 1000 queries for an abstract query graph without negation, 500 queries for an abstract query graph with negation; in FB15k, we sample 800 queries for an abstract query graph without negation, 400 queries for an abstract query graph with negation; in NELL, we sample 400 queries for an abstract query graph without negation, 100 queries for an abstract query graph with negation. As we have discussed in Appendix~\ref{app:sample negative query}, sample negative query is computationally costly, thus we sample less of them.

\begin{table}[t]
\caption{The number of abstract query graphs with one free variable. We denote $e$ as the number of existential variables and $c$ as the number of constant entities. SDAG represents the Simple Directed Acyclic Graph, Multi for multigraph, and Cyclic for the cyclic graph. Sum.($c$) and Sum.($e$) is the total number of queries with the number of constant entities / existential variables fixed.}
\label{tab: abstract query graph statistics of EFO1}
\centering
\footnotesize
\begin{tabular}{cccccccccc}
\toprule
\multirow{2}{*}{\diagbox{$c$}{$e$}} & 0   & \multicolumn{2}{c}{1} & \multicolumn{3}{c}
      {2} & \multirow{2}{*}{Sum.($c$)} & \multirow{2}{*}{Sum.} \\
      \cmidrule(lr){2-2} \cmidrule(lr){3-4} \cmidrule(lr){5-7}
                   & SDAG & SDAG      & Multi      & SDAG  & Multi & Cyclic &        &     \\
\midrule                  
1 &1&2&4&4&16&4&31&\multirow{3}{*}{251}   \\
2 &2&6&6&20&40&8&82  &  \\
3 &2&8&8&36&72&12&138 & \\\cmidrule(lr){2-8}
Sum.($e$)&5&16&18&60&128&24&\\
\bottomrule
\end{tabular}
\end{table}

\begin{table}[t]
\centering
\caption{The number of abstract query graphs with two free variables. The notation of $e$, $c$  SDAG, Multi, and Cyclic are the same as Table~\ref{tab: abstract query graph statistics of EFO1}. And "-" means that this type of abstract query graph is not included.} 
\label{tab: abstract query graph statistics of EFO2}
\footnotesize
\begin{tabular}{crrrrrrrrr}
\toprule
\multirow{2}{*}{\diagbox{$c$}{$e$}}  & \multicolumn{2}{c}{$e=0$} & \multicolumn{3}{c}{$e=1$} & \multicolumn{3}{c}{$e=2$} & \multirow{2}{*}{AVG.} \\
\cmidrule(lr){2-3} \cmidrule(lr){4-6} \cmidrule(lr){7-9} & SDAG      & Multi &  SDAG  & Multi & Cyclic &  SDAG  & Multi & Cyclic & \\
\midrule
$c=1$&1&2&7&18&4&6&32&26&96\\
$c=2$&4&4&20&36&8&38&108&64&282 \\
$c=3$&4&4&32&60&12&-&-&-&112\\
\midrule

\bottomrule
\end{tabular}
\end{table}

\section{Evaluation details}\label{app:evaluation detail}
\begin{algorithm}[htbp]
\caption{Embedding computation on the query graph.}\label{alg:QG embedding}
\begin{algorithmic}
\REQUIRE The query graph $G$.
\STATE Compute the ordering of the nodes as explained in Algorithm~\ref{alg:QG node ordering}.
\STATE Create a dictionary $E$ to store the embedding for each node in the query graph
\FOR{$i \gets 1$ to $n$ }
    \IF{node $u_i$ is a constant entity}
        \STATE The embedding of $u_i$, $E[i]$ is gotten from the entity embedding
    \ELSE
        \STATE Then we know node $u_i$ is either free variable or existential variable
        \STATE Compute the set of nodes $\{u_{i_j}\}_{j=1}^{t}$ that are previous to $i$ and adjacency to node $u_i$.
        \STATE Create a list to store projection embedding $L$.
        \FOR{$j \gets$ 1 to $t$}
            \STATE Find the relation $r$ between node $u_i$ and $u_{i_j}$, get the embedding of node $u_{i_j}$ as $E[i_j]$.
            \IF{$E[i_j]$ is not None}
                \IF{The edge between $u_i$ and $u_{i_J}$ is positive}
                    \STATE Compute the embedding of projection($E[i_j], r$), add it to the list $L$.
                \ELSE
                    \STATE Compute the embedding of the negation of the projection($E[i_j], r$), add it to the list $L$.
                \ENDIF
            \ENDIF
        \ENDFOR
        \IF{The list $L$ has no element}
            \STATE $E[i]$ is set to none.
        \ELSIF{The list $L$ has one element}
            \STATE $E[i]=L[0]$
        \ELSE
            \STATE Compute the embedding as the intersection of the embedding in the list $L$, and set $E[i]$ as the outcome.
        \ENDIF
    \ENDIF
\ENDFOR
\RETURN The embedding dictionary $E$ for each node in the query graph.
\end{algorithmic}
\end{algorithm}

\begin{algorithm}[htbp]
\caption{Node ordering on the abstract query graph.}\label{alg:QG node ordering}
\begin{algorithmic}
\REQUIRE The abstract query graph $\mathcal{G}=(V,E,f,g)$, $V$ consists $m$ nodes, $u_1,\cdots,u_m$.
\STATE Creates an empty list $L$ to store the ordering of the node.
\STATE Creates another two set $S_1$ and $S_2$ to store the nodes that are to be explored next, $S_1$ is for existential variables, and $S_2$ is for free variables.
\STATE Create a set $S_3$ to store explored nodes.
\FOR{$i \gets 1$ to $m$ }
    \IF{The type of node $f(u_i)$ is constant entity}
        \STATE list $L$ append the node $u_i$
        \FOR{Node $u_j$ that connects to $u_i$}
            \IF{$f(u_j)$ is existential variable}
                \STATE $u_j$ is added to set $S_1$
            \ELSE
                \STATE $u_j$ is added to set $S_2$
            \ENDIF
            \STATE $u_j$ is added to set $S_3$
        \ENDFOR
    \ENDIF
\ENDFOR
\WHILE{Not all node is included in $L$}
    \IF{Set $S_1$ is not empty}
        \STATE We sort the set $S_1$ by the sum of their distance to every free variable in $\mathcal{G}$, choose the most remote one, and if there is a tie, randomly choose one node, $u_i$ to be the next to explore.
        \STATE We remove $u_i$ from set $S_1$, add $u_i$ to $S_3$.
    \ELSE
        \STATE In this case, we know set $S_2$ is not empty because of the connectivity of $\mathcal{G}$.
        \STATE We randomly choose a node $u_i\in S_2$ to be the next node to explore.
        \STATE We remove $u_i$ from set $S_2$, add $u_i$ to $S_3$.
    \ENDIF
    \FOR{Node $u_j$ that connects to $u_i$}
        \IF{$u_j$ is not in $S_3$}
            \IF{$f(u_j)$ is existential variable}
                \STATE $u_j$ is added to set $S_1$.
            \ELSE
                \STATE $u_j$ is added to set $S_2$.
            \ENDIF
        \ENDIF
        \ENDFOR
    \STATE List $L$ append the node $u_i$
\ENDWHILE

\RETURN The list $L$ as the ordering of nodes in the whole abstract query graph $\mathcal{G}$
\end{algorithmic}
\end{algorithm}

We explain the evaluation protocol in detail for Section~\ref{sec: evaluation for multiple}.

Firstly, we explain the computation of common metrics, including Mean Reciprocal Rank(MRR) and HIT@K, given the full answer $\mathcal{A}$ in the whole knowledge graph and the observed answer $\mathcal{A}_o$ in the observed knowledge graph, we focus on the hard answer $\mathcal{A}_h$ as it requires more than memorizing the observed knowledge graph and serves as the indicator of the capability of reasoning.

Specifically, we rank each hard answer $a\in \mathcal{A}_h$ against all \textbf{non-answers} $\entity - \mathcal{A} - \mathcal{A}_o$, the reason is that we need to neglect other answers so that answers do not interfere with each other, finally, we get the ranking for $a$ as $r$. Then its MRR is $1/r$, and its HIT@k is $\mathbf{1}_{r\leq k}$, thus, the score of a query is the mean of the scores of every its hard answer. We usually compute the score for a query type (which corresponds to an abstract query graph) as the mean score of every query within this type.

As the marginal score and the multiply score have already been explained in Section~\ref{sec: evaluation for multiple}, we only mention one point that it is possible that every free variable does not have marginal hard answer.  Assume that for a query with two free variables, its answer set $\mathcal{A}=\{(a_1,a_2),(a_1,a_3),(a_4,a_2)\}$ and its observed answer set $\mathcal{A}_o=\{(a_1,a_3),(a_4,a_2)\}$. In this case, $a_1,a_4$ is not the marginal hard answer for the first free variable and $a_2,a_3$ is not the marginal hard answer for the second free variable, in general, no free variable has its own marginal hard answer.

Then we only discuss the joint metric, specifically, we only explain how to estimate the joint ranking by the individual ranking of each free variable. For each possible $k$-tuple $(a_1,\cdots,a_k)$, if $a_i$ is ranked as $r_i$ among the \textbf{whole} entity set $\entity$, we compute the score of this tuple as $\Sigma_{i=1}^{k} r_i$, then we sort the whole $\entity^k$ $k$-tuple by their score, for the situation of a tie, we just use the lexicographical order. After the whole joint ranking is got, we use the standard evaluation protocol that ranks each hard answer against all non-answers. It can be confirmed that this estimation method admits a closed-form solution for the sorting in $\entity^k$ space, thus the computation cost is affordable.

We just give the closed-form solution when there are two free variables: 

For the tuple $(r_1,r_2)$, the possible combinations that sum less than $r_1+r_2$ is $\binom{r_1+r_2 -1}{2}$, then, there is $r_1 - 1$ tuple that ranks before $(r_1,r_2)$ because of lexicographical order, thus, the final ranking for the tuple $(r_1,r_2)$ is just $\binom{r_1+r_2-1}{2} + r_1$, that can be computed efficiently. We note we adopt this kind of estimation since current CQA models fail to estimate the answer of multiple variables which have \textbf{combinatorial hardness}, therefore we have to use this way to \textbf{estimate} the joint ranking in the $\entity^{k}$ search space.

\section{Implementation details of CQA models}\label{app:implementation detail of CQA model}
In this section, we provide implementation details of CQA models that have been evaluated in our paper. 
For query embedding methods that rely on the operator tree, including BetaE~\cite{ren_beta_2020}, LogicE~\cite{luus_logic_2021}, and ConE~\cite{zhang_cone_2021}, we compute the ordering of nodes in the query graph in Algorithm~\ref{alg:QG node ordering}, then we compute the embedding for each node in the query graph as explained in Algorithm~\ref{alg:QG embedding}, the final embedding of every free node are gotten to be the predicted answer. Especially, the node ordering we got in Algorithm~\ref{alg:QG node ordering} coincides with the natural topology ordering induced by the directed acyclic operator tree, so we can compute the embedding in the same order as the original implementation. Then, in Algorithm~\ref{alg:QG embedding}, we implement each set operation in the operator tree, including intersection, negation, and set projection. By the merit of the Disjunctive Normal Form (DNF), the union is tackled in the final step. Thus, our implementation is able to coincide with the original implementation in the original dataset~\cite{ren_beta_2020}.

For CQD~\cite{arakelyan_complex_2021} and LMPNN~\cite{wang_logical_2023}, their original implementation does not require the operator tree, so we just use their original implementation. Specifically, in a query graph with multiple free variables, for CQD we predict the answer for each free variable individually as taking others free variables as existential variables, for LMPNN, we just got all embedding of nodes that represent free variables.

For FIT~\cite{yin_existential_2023}, though it is proposed to solve $\textsc{EFO}_1$ queries, it is computationally costly: it has a complexity of $O(\entity^2)$ in the acyclic graphs where $\entity$ is the set of entity of the knowledge graph, and the time complexity of FIT is even not polynomial in the cyclic graphs, the reason is that FIT degrades to enumeration to deal with cyclic graph. In our implementation, we further restrict FIT to at most enumerate 10 possible candidates for each node in the query graph, this practice has allowed FIT to be implemented in the dataset FB15k-237~\cite{toutanova_observed_2015}. However, it cost 20 hours to evaluate FIT on our $\efok$-CQA dataset while other models only need no more than two hours. Moreover, for larger knowledge graph, including NELL~\cite{carlson_toward_2010} and FB15k~\cite{bordes_translating_2013}, we have also encountered an out-of-memory error in a Tesla V100 GPU with 32G memory when implementing FIT, thus, we omit its result in these two knowledge graphs.

\section{Additional experiment result}\label{app: additional experiment result}

In this section, we offer more experiment results not available to be shown in the main paper. For the purpose of supplementation, we select some representative experiment results as the experiment results are extremely complex to be categorized and shown. we present the further benchmark result of the following situations: the situations of different knowledge graphs, including NELL and FB15k, whose results are provided in Appendix~\ref{app:further EFO1} and~\ref{app:EFO2 with more KG}; the situations of more constant entities since we only discuss when there are two constant entities in Table~\ref{tab: EFO2 result}, the results are provided in Appendix~\ref{app: EFO2 with more constants}, and finally, all queries(including the queries without marginal hard answers), in Appendix~\ref{app: EFO2 with more queries}.

We note that we have explained in Section~\ref{sec: evaluation for multiple} and Appendix~\ref{app:evaluation detail} that for a query with multiple free variables, some or all of the free variables may not have their marginal hard answer and thus the marginal metric can not be computed. Therefore, in the result shown in Table~\ref{tab: EFO2 result} in Section~\ref{sec: EFO2 result}, we only conduct evaluation on those queries that both of their free variables have marginal hard answers, and we offer the benchmark result of all queries in Appendix~\ref{app: EFO2 with more queries} where only two kinds of metrics are available.

\subsection{Further benchmark result of $k$=1 in more knowledge graphs}\label{app:further EFO1}
Firstly, we present the benchmark results when there is only one free variable, since the result in FB15k-237 is provided in Table~\ref{tab: EFO1 result}, we provide the result for other standard knowledge graphs, FB15k and NELL, their result is shown in Table~\ref{tab: FB15k EFO1 result} and Table~\ref{tab: NELL EFO1 result}, correspondingly. We note that FIT is out of memory with the two large graphs FB15k and NELL as explained in Appendix~\ref{app:implementation detail of CQA model} and we do not include its result. As FB15k and NELL are both reported to be easier than FB15k-237, the models have better performance. The trend and analysis are generally similar to our discussion in Section~\ref{sec: EFO1 result} with some minor, unimportant changes that LogicE~\cite{luus_logic_2021} has outperformed ConE~\cite{zhang_cone_2021} in the knowledge graph NELL, indicating one model may not perform identically well in all knowledge graphs.

\begin{table}[t]
\caption{MRR scores(\%) for inferring queries with one free variable on FB15k. The notation of $e$, $c$, SDAG, Multi, Cyclic, AVG.($c$) and AVG.($e$) are the same as Table~\ref{tab: EFO1 result}.}
\label{tab: FB15k EFO1 result}
\centering
\footnotesize
\begin{tabular}{cccccccccc}
\toprule
      Model            & \diagbox{$c$}{$e$} & 0   & \multicolumn{2}{c}{1} & \multicolumn{3}{c}
      {2} & \multirow{2}{*}{AVG.($c$)} & \multirow{2}{*}{AVG.} \\
      \cmidrule(lr){3-3} \cmidrule(lr){4-5} \cmidrule(lr){6-8}
                  &   & SDAG & SDAG      & Multi      & SDAG  & Multi & Cyclic &        &     \\
\midrule                  
\multirow{4}{*}{BetaE} & 1 &38.6&30.4&29.2&21.7&21.7&24.1&24.3&\multirow{3}{*}{34.0}   \\
                  & 2 &49.7&34.0&37.2&28.3&29.2&35.5&31.0&    \\
                  & 3 &63.5&46.4&48.6&33.9&36.1&45.8&38.1&  \\\cmidrule(lr){2-8}
                  & AVG.($e$)&63.5&46.4&48.6&33.9&36.1&45.8&38.1&\\
\midrule                  
\multirow{4}{*}{LogicE} & 1 &46.0&33.8&32.1&23.3&22.8&25.6&26.2&\multirow{3}{*}{35.6}   \\
                  & 2 &51.2&35.9&39.0&30.6&30.5&36.9&32.7&   \\
                  & 3 &64.5&48.6&49.8&35.4&37.5&47.7&39.6&\\  \cmidrule(lr){2-8}
                  & AVG.($e$)&54.9&41.7&42.3&32.8&33.4&40.4&\\
\midrule                  
\multirow{4}{*}{ConE} & 1 &52.5&35.8&34.9&25.9&25.9&29.5&29.3&\multirow{4}{*}{39.5}   \\
                  & 2 &57.0&40.0&43.4&33.2&34.2&40.8&36.3&    \\
                  & 3 &70.6&53.1&55.3&39.3&41.8&52.5&43.9&  \\ \cmidrule(lr){2-8}
                  & AVG.($e$)&61.0&45.6&46.8&36.1&37.4&44.8&\\
\midrule        
\multirow{4}{*}{CQD} 
& 1 &74.6&36.1&32.7&17.6&16.7&25.4&23.7&\multirow{3}{*}{37.2}   \\
                  & 2 &52.2&35.2&40.9&29.2&31.5&39.2&33.2    \\
                  & 3 &53.3&32.4&33.1&21.7&21.6&37.4&24.8  \\\cmidrule(lr){2-8}
                  & AVG.($e$)&59.4&41.5&44.6&33.3&35.3&43.3\\
\midrule
\multirow{4}{*}{LMPNN} & 1 &63.7&39.9&35.3&28.7&26.4&28.7&30.7&\multirow{4}{*}{37.7}   \\
                  & 2 &65.0&41.9&38.8&34.4&31.7&38.4&35.1 &   \\
                  & 3 &79.8&54.0&49.5&38.9&37.1&48.0&40.8 & \\ \cmidrule(lr){2-8}
                  & AVG.($e$)&70.2&47.4&42.8&36.6&34.1&41.6& \\
\bottomrule
\end{tabular}
\vspace{-1em}
\end{table}

\begin{table}[t]
\caption{MRR scores(\%) for inferring queries with one free variable on NELL. The notation of $e$, $c$, SDAG, Multi, Cyclic, AVG.($c$) and AVG.($e$) are the same as Table~\ref{tab: EFO1 result}.}
\label{tab: NELL EFO1 result}
\centering
\footnotesize
\begin{tabular}{cccccccccc}
\toprule
      Model            & \diagbox{$c$}{$e$} & 0   & \multicolumn{2}{c}{1} & \multicolumn{3}{c}
      {2} & \multirow{2}{*}{AVG.($c$)} & \multirow{2}{*}{AVG.} \\
      \cmidrule(lr){3-3} \cmidrule(lr){4-5} \cmidrule(lr){6-8}
                  &   & SDAG & SDAG      & Multi      & SDAG  & Multi & Cyclic &        &     \\
\midrule                  
\multirow{4}{*}{BetaE} & 1 &13.9&26.4&35.0&8.6&14.9&19.1&17.5&\multirow{3}{*}{33.6}   \\
                  & 2 &58.8&31.5&43.8&22.4&30.6&34.7&30.7    \\
                  & 3 &78.8&48.6&58.3&29.6&39.0&47.0&39.5  \\\cmidrule(lr){2-8}
                  & AVG.($e$)&53.1&38.5&48.3&25.2&33.3&38.2\\
\midrule                  
\multirow{4}{*}{LogicE} & 1 &18.3&29.2&39.6&12.1&19.0&20.4&21.1&\multirow{3}{*}{36.9}   \\
                  & 2 &63.5&34.4&47.3&26.4&34.0&37.6&34.2 &   \\
                  & 3 &79.6&51.2&59.3&33.1&42.2&50.1&42.6  &\\  \cmidrule(lr){2-8}
                  & AVG.($e$)&56.3&41.3&50.9&28.8&36.7&41.0&\\
\midrule                  
\multirow{4}{*}{ConE} & 1 &16.7&26.9&36.6&11.1&16.9&22.3&19.6&\multirow{4}{*}{36.6}   \\
                  & 2 &60.5&33.6&46.6&25.3&33.1&40.1&33.6    \\
                  & 3 &79.9&50.6&59.2&33.2&42.2&52.6&42.8  \\ \cmidrule(lr){2-8}
                  & AVG.($e$)&54.9&40.3&50.0&28.4&36.2&43.4&\\
\midrule        
\multirow{4}{*}{CQD} 
& 1 &22.3&30.6&37.3&13.3&17.9&20.7&20.9&\multirow{3}{*}{38.2}   \\
                  & 2 &59.8&34.0&45.2&28.8&35.4&38.9&35.3    \\
                  & 3 &62.7&48.8&59.9&36.4&44.1&52.6&44.3  \\\cmidrule(lr){2-8}
                  & AVG.($e$)&50.1&40.2&49.9&31.6&38.1&42.7\\
\midrule
\multirow{4}{*}{LMPNN} & 1 &20.7&29.8&33.3&13.4&16.5&21.8&19.8&\multirow{4}{*}{35.1}   \\
                  & 2 &63.5&35.4&43.3&27.0&30.2&37.6&32.3 &   \\
                  & 3 &80.8&50.7&56.0&33.6&39.2&47.6&40.7 & \\ \cmidrule(lr){2-8}
                  & AVG.($e$)&57.4&41.5&46.7&29.4&33.6&40.0& \\
\bottomrule
\end{tabular}
\end{table}

\subsection{Further benchmark result for $k$=2 in more knowledge graphs}\label{app:EFO2 with more KG}
Then, similar to Section~\ref{sec: EFO2 result}, we provide the result for other standard knowledge graphs, FB15k and NELL, when the number of constant entities is fixed to two, their result is shown in Table~\ref{tab: FB15k EFO2 result} and Table~\ref{tab: NELL EFO2 result}, correspondingly.

We note that though in some breakdowns, the marginal score is over 90 percent, almost close to 100 percent, the joint score is pretty slow, which further corroborates our findings that joint metric is significantly harder and more challenging in Section~\ref{sec: EFO2 result}.

\begin{table}[t]
\centering
\caption{HIT@10 scores(\%) of three different types for answering queries with two free variables on FB15k. The constant number is fixed to be two. The notation of $e$, SDAG, Multi, and Cyclic is the same as Table~\ref{tab: EFO2 result}.}
\label{tab: FB15k EFO2 result}
\footnotesize
\begin{tabular}{ccrrrrrrrrr}
\toprule
\multirow{2}{*}{Model}  & \multirow{2}{*}{\shortstack[c]{HIT@10\\ Type}}  & \multicolumn{2}{c}{$e=0$} & \multicolumn{3}{c}{$e=1$} & \multicolumn{3}{c}{$e=2$} & \multirow{2}{*}{AVG.} \\
\cmidrule(lr){3-4} \cmidrule(lr){5-7} \cmidrule(lr){8-10} & & SDAG      & Multi &  SDAG  & Multi & Cyclic &  SDAG  & Multi & Cyclic & \\
\midrule
\multirow{3}{*}{BetaE} 
& Marginal&76.9&77.2&68.9&69.3&75.1&55.0&57.4&73.6&63.6 \\
& Multiply&41.7&41.6&31.7&31.0&38.7&25.2&25.9&36.1&29.7\\
& Joint&11.6&13.7&8.7&8.6&17.8&4.9&5.4&14.3&8.4\\
\midrule
\multirow{3}{*}{LogicE} 
& Marginal&82.9&80.9&73.6&72.9&76.6&58.9&60.7&75.7&66.9 \\
& Multiply&47.5&45.0&36.3&34.1&40.4&28.5&29.0&38.0&32.7\\
& Joint&12.7&13.9&10.0&9.9&19.2&6.1&6.5&15.9&9.6 \\
\midrule
\multirow{3}{*}{ConE} 
& Marginal&84.1&84.8&76.5&76.3&81.4&61.8&63.8&79.7&70.2 \\
& Multiply&48.7&48.1&37.7&35.9&44.2&29.9&30.4&41.4&34.6\\
& Joint&14.2&15.6&10.3&10.4&20.6&6.2&6.6&16.9&10.1 \\
\midrule
\multirow{3}{*}{CQD} 
& Marginal&73.8&76.8&69.0&71.9&76.3&51.1&54.4&77.0&62.9 \\
& Multiply  &45.0&46.6&37.4&36.9&43.9&28.1&29.2&41.9&34.0\\
& Joint   &17.1&19.0&13.1&13.0&20.6&7.7&8.6&18.1&11.9  \\
\midrule
\multirow{3}{*}{LMPNN} 
& Marginal&89.2&80.1&80.3&78.2&84.2&65.6&63.7&80.2&71.3 \\
& Multiply   &56.6&50.5&45.7&42.4&49.0&37.6&34.8&44.6&39.7\\
& Joint   &18.9&17.2&12.9&12.4&22.4&8.0&7.5&16.9&11.2  \\
\bottomrule
\end{tabular}
\vspace{-1em}
\end{table}

\begin{table}[t]
\centering
\caption{HIT@10 scores(\%) of three different types for answering queries with two free variables on NELL. The constant number is fixed to be two. The notation of $e$, SDAG, Multi, and Cyclic is the same as Table~\ref{tab: EFO2 result}.}
\label{tab: NELL EFO2 result}
\footnotesize
\begin{tabular}{ccrrrrrrrrr}
\toprule
\multirow{2}{*}{Model}  & \multirow{2}{*}{\shortstack[c]{HIT@10\\ Type}}  & \multicolumn{2}{c}{$e=0$} & \multicolumn{3}{c}{$e=1$} & \multicolumn{3}{c}{$e=2$} & \multirow{2}{*}{AVG.} \\
\cmidrule(lr){3-4} \cmidrule(lr){5-7} \cmidrule(lr){8-10} & & SDAG      & Multi &  SDAG  & Multi & Cyclic &  SDAG  & Multi & Cyclic & \\
\midrule
\multirow{3}{*}{BetaE} 
& Marginal&81.3&95.9&72.8&85.5&79.9&57.2&66.7&77.0&71.2 \\
& Multiply&48.2&56.7&41.3&46.1&47.6&33.1&36.5&42.9&39.6\\
& Joint&19.2&31.8&21.2&26.5&21.7&13.8&17.5&18.5&18.8\\
\midrule
\multirow{3}{*}{LogicE} 
& Marginal&87.1&99.8&81.0&91.8&83.2&65.7&74.0&81.0&77.7 \\
& Multiply&52.5&60.3&47.6&51.7&50.2&39.4&42.6&46.0&44.8\\
& Joint&21.1&32.8&25.4&30.5&23.3&18.0&21.5&20.5&22.3\\
\midrule
\multirow{3}{*}{ConE} 
& Marginal&82.6&96.4&76.0&87.8&88.1&60.0&69.3&83.0&74.7 \\
& Multiply&48.7&56.9&41.9&46.3&52.2&34.5&38.1&47.7&41.7\\
& Joint&17.0&30.9&19.3&25.0&24.9&12.9&17.2&20.3&18.8 \\
\midrule
\multirow{3}{*}{CQD} 
& Marginal&79.5&96.3&83.2&92.2&83.5&65.8&75.7&84.8&79.4 \\
& Multiply  &49.2&57.8&51.1&53.1&51.4&40.6&45.1&50.6&47.4\\
& Joint   &23.0&38.0&29.7&34.2&26.4&21.4&25.4&24.0&26.0  \\
\midrule
\multirow{3}{*}{LMPNN} 
& Marginal&88.5&96.6&81.5&90.9&85.3&65.0&70.7&83.1&76.7 \\
& Multiply   &55.7&62.4&50.3&53.3&54.0&40.8&42.6&50.3&46.5\\
& Joint   &23.4&36.4&25.5&29.4&24.0&16.6&19.7&21.5&21.5  \\
\bottomrule
\end{tabular}
\vspace{-1em}
\end{table}

\subsection{Further benchmark result for $k$=2 with more constant numbers.}\label{app: EFO2 with more constants}

As the experiment in Section~\ref{sec: EFO2 result} only contains the situation where the number of constant entity is fixed as one, we offer the further experiment result in Table~\ref{tab: FB15k-237 EFO2 result with one constant}.

The result shows that models perform worse with fewer constant variables when compares to the result in Table~\ref{tab: EFO2 result}, this observation is the same as the previous result with one free variable that has been discussed in Section~\ref{sec: EFO1 result}.

\begin{table}[t]
\centering
\caption{HIT@10 scores(\%) of three different types for answering queries with two free variables on FB15k-237. The constant number is fixed to be one. The notation of $e$, SDAG, Multi, and Cyclic is the same as Table~\ref{tab: EFO2 result}.}
\label{tab: FB15k-237 EFO2 result with one constant}
\footnotesize
\begin{tabular}{ccrrrrrrrrr}
\toprule
\multirow{2}{*}{Model}  & \multirow{2}{*}{\shortstack[c]{HIT@10\\ Type}}  & \multicolumn{2}{c}{$e=0$} & \multicolumn{3}{c}{$e=1$} & \multicolumn{3}{c}{$e=2$} & \multirow{2}{*}{AVG.} \\
\cmidrule(lr){3-4} \cmidrule(lr){5-7} \cmidrule(lr){8-10} & & SDAG      & Multi &  SDAG  & Multi & Cyclic &  SDAG  & Multi & Cyclic & \\
\midrule
\multirow{3}{*}{BetaE} 
& Marginal&37.5&29.7&33.4&28.1&35.6&30.0&25.9&41.2&31.2 \\
& Multiply&18.9&13.7&15.3&10.3&15.2&17.7&13.3&17.2&14.3\\
& Joint&0.9&1.1&1.4&0.9&3.3&1.1&0.9&3.9&1.7\\
\midrule
\multirow{3}{*}{LogicE} 
& Marginal&40.6&30.7&36.0&29.1&34.6&29.8&25.3&41.5&31.4 \\
& Multiply&21.1&14.3&17.2&10.9&16.3&17.8&13.3&17.5&14.7\\
& Joint&1.4&1.4&1.6&0.9&3.7&1.4&1.0&4.3&1.9 \\
\midrule
\multirow{3}{*}{ConE} 
& Marginal&40.8&32.4&37.3&30.4&40.7&31.1&26.9&45.0&33.5 \\
& Multiply&22.1&15.2&18.4&11.7&19.3&18.5&14.8&20.9&16.5\\
& Joint&1.4&1.0&1.7&1.0&4.3&1.4&1.0&4.4&2.0 \\
\midrule
\multirow{3}{*}{CQD} 
& Marginal&73.8&76.8&69.0&71.9&76.3&51.1&54.4&77.0&62.9 \\
& Multiply  &23.3&9.1&18.5&9.2&16.2&14.6&9.2&19.1&12.9\\
& Joint   &1.5&0.6&2.0&1.1&3.4&1.5&0.9&4.4&1.9  \\
\midrule
\multirow{3}{*}{LMPNN} 
& Marginal&39.0&27.6&40.0&29.5&39.3&30.6&24.8&42.7&32.0 \\
& Multiply   &25.1&13.9&24.3&13.3&21.6&20.0&14.0&21.1&17.1\\
& Joint   &1.6&1.3&2.5&1.3&3.9&1.5&1.0&4.0&2.0  \\
\bottomrule
\end{tabular}
\vspace{-1em}
\end{table}

\subsection{Further benchmark result for $k$=2 including all queries}\label{app: EFO2 with more queries}
Finally, as we have explained in Section~\ref{sec: evaluation for multiple} and Appendix~\ref{app:evaluation detail}, there are some valid $\efok$ queries without marginal hard answers when $k>1$. Thus, there is no way to calculate the marginal scores, all our previous experiments are therefore only conducted on those queries that all their free variables have marginal hard answers. In this section, we only present the result of the Multiply and Joint score, as they can be computed for any valid $\efok$ queries, and therefore this experiment is conducted on the whole $\efok$-CQA dataset.

We follow the practice in Section~\ref{sec: EFO2 result} that fixed the number of constant entities as two, as the impact of constant entities is pretty clear, which has been further corroborated in Appendix~\ref{app: EFO2 with more constants}. The experiments are conducted on all three knowledge graphs, FB15k-237, FB15k, and NELL, the result is shown in Table~\ref{tab: FB15k-237 EFO2 result with all queries}, Table~\ref{tab: FB15k EFO2 result with all queries}, and Table~\ref{tab: NELL EFO2 result with all queries}, correspondingly.

Interestingly, comparing the result in Table~\ref{tab: EFO2 result} and  Table~\ref{tab: FB15k-237 EFO2 result with all queries}, the multiple scores actually increase through the joint scores are similar. This may be explained by the fact that if one free variable has no marginal hard answer, then it can be easily predicted, leading to a better performance for the whole query.

\begin{table}[t]
\centering
\caption{HIT@10 scores(\%) of two different types for answering queries with two free variables on FB15k-237(including queries without the marginal hard answer). The constant number is fixed to be two. The notation of $e$, SDAG, Multi, and Cyclic is the same as Table~\ref{tab: EFO2 result}.}
\label{tab: FB15k-237 EFO2 result with all queries}
\footnotesize
\begin{tabular}{ccrrrrrrrrr}
\toprule
\multirow{2}{*}{Model}  & \multirow{2}{*}{\shortstack[c]{HIT@10\\ Type}}  & \multicolumn{2}{c}{$e=0$} & \multicolumn{3}{c}{$e=1$} & \multicolumn{3}{c}{$e=2$} & \multirow{2}{*}{AVG.} \\
\cmidrule(lr){3-4} \cmidrule(lr){5-7} \cmidrule(lr){8-10} & & SDAG      & Multi &  SDAG  & Multi & Cyclic &  SDAG  & Multi & Cyclic & \\
\midrule
\multirow{3}{*}{BetaE} 
& Multiply&29.1&29.1&18.3&37.5&10.4&28.0&93.6&74.6&24.1\\
& Joint&2.1&2.2&1.7&3.0&2.4&1.8&5.8&14.2&4.6\\
\midrule
\multirow{3}{*}{LogicE} 
& Multiply&31.6&32.9&19.8&39.6&10.9&28.7&96.3&73.8&25.4\\
& Joint&2.6&2.5&2.1&3.1&2.5&2.2&6.4&15.6&5.0 \\
\midrule
\multirow{3}{*}{ConE} 
& Multiply&32.6&31.9&20.5&41.0&12.6&29.0&99.7&86.8&27.0\\
& Joint&3.0&2.1&1.9&3.3&2.7&2.2&6.6&16.8&5.4 \\
\midrule
\multirow{3}{*}{CQD} 
& Multiply  &34.5&23.4&22.3&36.8&10.6&26.4&75.3&77.3&25.6\\
& Joint   &2.9&1.4&2.1&3.3&2.3&2.0&5.0&15.0&5.6  \\
\midrule
\multirow{3}{*}{LMPNN} 
& Multiply   &36.8&29.3&27.5&45.8&13.9&31.2&97.0&86.5&27.9\\
& Joint   &2.7&2.2&2.7&3.9&2.5&2.1&5.8&14.6&5.0  \\
\midrule
\multirow{3}{*}{FIT} 
& Multiply &41.5&44.4&28.9&56.8&10.2&39.4&139.7&100.3&35.0\\
& Joint&2.4&2.3&2.1&3.4&1.6&2.2&7.4&15.4&5.9\\
\bottomrule
\end{tabular}
\end{table}

\begin{table}[t]
\centering
\caption{HIT@10 scores(\%) of two different types for answering queries with two free variables on FB15k(including queries without the marginal hard answer). The constant number is fixed to be two. The notation of $e$, SDAG, Multi, and Cyclic is the same as Table~\ref{tab: EFO2 result}.}
\label{tab: FB15k EFO2 result with all queries}
\footnotesize
\begin{tabular}{ccrrrrrrrrr}
\toprule
\multirow{2}{*}{Model}  & \multirow{2}{*}{\shortstack[c]{HIT@10\\ Type}}  & \multicolumn{2}{c}{$e=0$} & \multicolumn{3}{c}{$e=1$} & \multicolumn{3}{c}{$e=2$} & \multirow{2}{*}{AVG.} \\
\cmidrule(lr){3-4} \cmidrule(lr){5-7} \cmidrule(lr){8-10} & & SDAG      & Multi &  SDAG  & Multi & Cyclic &  SDAG  & Multi & Cyclic & \\
\midrule
\multirow{3}{*}{BetaE} 
& Multiply&42.1&57.2&26.5&66.5&15.5&34.6&134.9&100.0&35.0\\
& Joint&6.6&9.4&4.5&10.2&4.6&4.3&16.7&26.0&9.2\\
\midrule
\multirow{3}{*}{LogicE} 
& Multiply&48.2&65.6&31.0&71.6&16.8&37.8&143.9&105.8&38.1\\
& Joint&7.5&11.2&5.6&12.5&5.3&5.6&20.4&28.5&10.5\\
\midrule
\multirow{3}{*}{ConE} 
& Multiply&50.2&72.2&32.8&74.6&18.3&38.3&149.3&114.3&40.4\\
& Joint&6.8&10.0&5.2&12.5&5.5&5.2&19.4&30.4&11.0\\
\midrule
\multirow{3}{*}{CQD} 
& Multiply  &48.1&55.9&31.9&69.0&15.8&29.5&93.5&103.2&37.6\\
& Joint   &9.4&11.4&6.6&14.8&4.8&5.5&17.5&27.2&12.0\\
\midrule
\multirow{3}{*}{LMPNN} 
& Multiply&58.4&79.5&43.1&94.6&21.3&40.9&146.2&135.9&45.0\\
& Joint&8.6&12.9&6.8&15.6&6.2&5.4&19.3&31.7&11.6\\
\bottomrule
\end{tabular}
\vspace{-1em}
\end{table}

\begin{table}[t]
\centering
\caption{HIT@10 scores(\%) of two different types for answering queries with two free variables on NELL(including queries without the marginal hard answer). The constant number is fixed to be two. The notation of $e$, SDAG, Multi, and Cyclic is the same as Table~\ref{tab: EFO2 result}.}
\label{tab: NELL EFO2 result with all queries}
\footnotesize
\begin{tabular}{ccrrrrrrrrr}
\toprule
\multirow{2}{*}{Model}  & \multirow{2}{*}{\shortstack[c]{HIT@10\\ Type}}  & \multicolumn{2}{c}{$e=0$} & \multicolumn{3}{c}{$e=1$} & \multicolumn{3}{c}{$e=2$} & \multirow{2}{*}{AVG.} \\
\cmidrule(lr){3-4} \cmidrule(lr){5-7} \cmidrule(lr){8-10} & & SDAG      & Multi &  SDAG  & Multi & Cyclic &  SDAG  & Multi & Cyclic & \\
\midrule
\multirow{3}{*}{BetaE} 
& Multiply&21.2&47.3&22.0&51.9&14.7&24.1&80.5&79.7&33.4\\
& Joint&4.2&19.6&6.8&19.1&5.1&6.8&26.7&24.0&14.1\\
\midrule
\multirow{3}{*}{LogicE} 
& Multiply&26.6&52.8&28.8&63.4&16.0&32.8&103.1&88.5&38.9\\
& Joint&3.8&21.5&9.7&26.0&5.9&11.5&36.9&27.3&16.5\\
\midrule
\multirow{3}{*}{ConE} 
& Multiply&25.3&51.4&23.9&53.9&16.9&27.3&90.7&90.6&36.7\\
& Joint&3.4&20.2&6.4&17.0&6.1&7.2&27.0&27.1&14.2\\
\midrule
\multirow{3}{*}{CQD} 
& Multiply&30.3&48.9&30.6&64.3&15.9&33.1&88.9&91.2&40.9\\
& Joint&4.4&21.9&9.8&27.5&5.6&12.0&37.6&28.1&18.0\\
\midrule
\multirow{3}{*}{LMPNN} 
& Multiply&33.4&58.3&33.7&65.3&19.4&30.7&85.1&105.0&41.8\\
& Joint&4.4&23.7&10.0&21.9&5.8&8.2&23.2&28.8&15.7\\
\bottomrule
\end{tabular}
\end{table}

\end{document}